\documentclass[times,onecolumn,preprint]{elsarticle}

\usepackage{prletters}
\usepackage{framed,multirow}

\usepackage{amssymb}
\usepackage{latexsym}

\usepackage{url}
\usepackage{xcolor}
\definecolor{newcolor}{rgb}{.8,.349,.1}

\journal{Pattern Recognition Letters}

\usepackage{subfigure}
\usepackage{booktabs} 
\usepackage{amsmath,graphicx}
\usepackage{amsthm}

\def\x{{\mathbf x}}
\def\y{{\mathbf y}}
\def\yi {y} 

\def \classzero {{1}}
\def \classone {{2}}
\def \classc {{c}}

\def\X{{\mathbf X}}

\def\mm{{\mathbf m}}

\def\X{{\mathbf X}}
\def\S{{\mathbf S}}

\def\L{{\cal L}}

\def\mub{{\boldsymbol{\mu}}}
\def\thb{{\boldsymbol{\theta}}}

\def\Sigmab{{\boldsymbol{\Sigma}}}
\def\R{{\mathbb{R}}}

\def \zb {{\mathbf z}}
\def \Wb {{\mathbf U}}
\def \avar {{\mathbf{w}}}
\def \bvar {{b}}

\def \prior {{p}}

\def \dimP {{P}}
\def \cbauc {{\textrm{CBAUC}}}

\DeclareMathOperator{\sgn}{sgn}
\DeclareMathOperator{\Tr}{Tr}
\newcommand{\jussi}[1]{{\color{black}{#1}}}
\newcommand{\heikki}[1]{{\color{black}{#1}}}
\newcommand{\sakira}[1]{{\color{black}{#1}}}

\newtheorem{definition}{Definition}
\newtheorem{proposition}{Proposition}
\newtheorem{lemma}{Lemma}

\begin{document}

\ifpreprint
  \setcounter{page}{1}
\else
  \setcounter{page}{1}
\fi

\begin{frontmatter}

\title{Bayesian Receiver Operating Characteristic Metric for Linear Classifiers}

\author[1]{Syeda Sakira \snm{Hassan}\corref{cor1}}
\cortext[cor1]{Corresponding author.}

\ead{sakira.hassan@gmail.com}
\author[1]{Heikki \snm{Huttunen}}
\author[1]{Jari \snm{Niemi}}
\author[2]{Jussi \snm{Tohka}}

\address[1]{Tampere University, Korkeakoulunkatu 1, Tampere, 33720, Finland}
\address[2]{\jussi{A.I. Virtanen Institute for Molecular Sciences,} University of Eastern Finland, Neulaniementie 2, 70211, Kuopio, Finland}


\begin{abstract}
We propose a novel classifier accuracy metric: the Bayesian Area Under the Receiver Operating Characteristic Curve (CBAUC). The method estimates the area under the ROC curve and is related to the recently proposed Bayesian Error Estimator. The metric can assess the quality of a classifier using only the training dataset without the need for computationally expensive cross-validation. We derive a closed-form solution of the proposed accuracy metric for any linear binary classifier under the Gaussianity assumption, and study the accuracy of the proposed estimator using simulated and real-world data. These experiments confirm that the closed-form CBAUC is both faster and more accurate than conventional AUC estimators.
\end{abstract}


\end{frontmatter}

\textit{Preprint submitted to Pattern Recognition letters 10 Oct, 2018}
\\
\textit{ Preprint accepted to Pattern Recognition letters 23 Jul, 2019}


\section{INTRODUCTION}
\label{sec:intro}
Supervised classification attempts to learn a model from a training set $\{(\x_1, \yi_1), \ldots (\x_N, \yi_N)\}$, $\x_k\in\R^P$, for predicting the class label of a new test sample. Besides the actual training of a classifier, the estimation of its accuracy is a critically important step, as it is the key tool for selection of the model structure and the related hyperparameters~\citep{alpaydin2014introduction, hastie2009elements, Duda01}.

The conventional approaches (cross-validation, bootstrap, bolstering) for error estimation are based on counting ~\citep{stone1974cross, dougherty2010performance, efron1992bootstrap, efron1983estimating, efron1997improvements, braga2004bolstered, Glick1978}. The most popular of these methods---cross-validation---splits the training data into distinct subsets for training and testing, and the number of erroneously predicted samples in the test partition is counted. Based on the test partition, it is also possible to calculate alternative accuracy measures, such as the Area Under the Receiver Operating Characteristics Curve (AUC). 

The cross-validation, in particular, has a number of drawbacks: The accuracy estimate depends on the particular split of the data\footnote{This can be alleviated at a considerable computational expense by re-running the cross-validation using several splits of the data.}, the approach is time consuming, and the resulting error estimate may have a large variance~\citep{dougherty2010performance}. In particular, the latter problem, which is shared by other counting based error estimates, has been documented already four decades ago~\citep{Glick1978}, but it is still often dismissed~\citep{dougherty2010performance}.

In addition to this, classification error rate may not be as efficient as AUC in quantifying the discrimination between classes~\citep{ferri2009experimental, jeni2013facing}. For example, classification error rate ignores imbalanced data and assumes data is equally distributed among classes.  
In binary classification, if one of the classes (minority class) is represented by a small number of observations compared to the other (majority class) class, then the dataset is said to be imbalanced~\citep{he2008learning}. The performance of the classifiers is often biased towards the majority class and this can misclassify the observations of minority class in real-world applications, such as fraud detection, diagnosis of rare diseases~\citep{cohen2006learning}. Moreover, in model selection, AUC is more appropriate choice to measure the ranking quality of a classifier~\citep{bradley1997use, NIPS2003_2518}. Thus, we are interested in finding a {deterministic}, {accurate} and {fast} approach for estimating the AUC.

Recently, a {Bayesian minimum mean-square estimator for classification error} was proposed  for binary classification~\citep{Dalton2011BayesianDiscrete,dalton2011bayesian} and extended to multiclass model selection~\citep{Huttunen2015BEE}. The Bayesian error estimator defines the problem in a Bayesian framework, and a closed-form expression can be derived for the posterior expectation of the classification error in the binary classification case for linear~\citep{dalton2011bayesian} and discrete classifiers~\citep{Dalton2011BayesianDiscrete}. The accuracy has been shown to be superior, especially in small sample settings as long as the prior assumptions of the underlying distribution are not severely violated.

The Bayesian Error Estimator is attractive because the error is estimated directly from the training set, and no iterative resampling or splitting operations are required. This results in a significant speedup, since the classifier is trained only once. For example, the 5-fold cross-validation includes five training iterations on partial data and one on all training data, while the Bayesian error estimator requires only a single training step with all data.

In this paper, we extend the Bayesian approach to estimating the AUC instead of the classification error, and derive a closed-form equation for evaluating this quantity. \jussi{We assume that the class conditional distribution are Gaussian and the true covariance matrix is the same for the two classes.} We abbreviate the the resulting estimator as CBAUC. The Bayesian estimation of AUC has been studied in~\citep{dalton2015ROC} by deriving optimal estimators for the true positive rate (TPR) and false positive rate (FPR), and iteratively sampling the TPR-FPR space. The proposed closed-form solution avoids the sampling step. In the following, we will call the approach in~\citep{dalton2015ROC} \textit{Empirical Bayesian AUC} (EBAUC).

%

\section{BAYESIAN AUC}
\label{sec:methods}
In this section, we will derive an analytic solution for the Bayesian Area Under the ROC Curve estimator. 
%
We consider a two-class classification problem with samples $\X = (\x_1,\x_2,\ldots, \x_N)$ and class labels $\y=(y_1, y_2,\ldots, y_N)$ with $\x_i \in\R^P$ and $y_i \in \{1,2\}$ for $i=1,\ldots, N$. We assume $P$-dimensional Gaussian densities with parameters $\mub_\classzero$, $\mub_\classone$, $\Sigmab$, where $\mub_\classzero\in\R^P$ and $\mub_\classone\in\R^P$ are the centers of the two Gaussian classes and $\Sigmab\in\R^{P\times P}$ is their common (invertible) covariance matrix. 
Finally, given the distribution parameters $\thb_\classzero = \{ \mub_\classzero, \Sigmab \}$ and $\thb_\classone = \{ \mub_\classone, \Sigmab \}$, we denote the densities for the samples 
as 
$f_{\thb_1}(\x_i)$ for the class $y_i=1$ and $f_{\thb_2}(\x_i)$ for $y_i=2$. 

In our Bayesian framework, we define a prior of the distribution parameters, $\prior(\thb_\classzero, \thb_\classone) = \prior(\Sigmab, \mub_\classzero, \mub_\classone) = \prior(\Sigmab) \, \prior(\mub_\classzero|\Sigmab) \, \prior(\mub_\classone|\Sigmab)$ (assuming independence). Then, the posterior density for the hyperparameters~\citep[Eq.~1]{dalton2011bayesian} becomes 

\begin{equation}
\prior^*(\thb_\classzero, \thb_\classone\mid \X, \y) \propto \prior(\thb_\classzero, \thb_\classone) \prod_{y_i=\classzero} f_{\thb_\classzero}(\x_i) \prod_{y_i=\classone} f_{\thb_\classone}(\x_i).
\label{eq:posterior}
\end{equation}
If the distribution parameters $\thb_1$ and $\thb_2$ were known, the population AUC for a linear classifier with weights $\avar\in \R^P$ could be written as~\citep[p.~1415]{demler2011equivalence}
\begin{equation}
\text{AUC}(\avar \mid \thb_1, \thb_2) = \Phi \left( \frac{\avar^T ( \mub_\classone - \mub_\classzero )}{\sqrt{2 \avar^T \Sigmab \avar}} \right).
\label{eq:auroc}
\end{equation}
Here, $\Phi$ is the Gaussian (cumulative) distribution function. With these preliminaries, we can state the definition of Bayesian AUC, which we call $\cbauc$ (Closed-form Bayesian AUC) to emphasize the difference to the EBAUC.
\begin{definition}
The \emph{Bayesian Area Under the Receiver Operating Characteristic Curve} of a linear classifier with weights $\avar\in \R^P$ classifying $P$-dimensional samples from two Gaussian distributions with parameters $\thb_\classzero = \{ \mub_\classzero, \Sigmab \}$ and $\thb_\classone = \{ \mub_\classone, \Sigmab \}$ is the posterior expectation of the AUC in Eq.~\ref{eq:auroc}: 
\begin{equation}
\begin{aligned}
& \emph{CBAUC}(\avar) \\
& = \int \, \emph{AUC}(\avar \mid \thb_1,\thb_2) \, \prior^*(\thb_\classzero, \thb_\classone \mid \X,\y) \, \mathrm{d}\thb_\classzero \mathrm{d}\thb_\classone.
\end{aligned}
\label{eq:bee_auc_1}
\end{equation}
\end{definition}
The above integral is similar to the one in the definition of Bayesian MMSE error estimator~\citep[Eq.~3]{dalton2011bayesian}, except now we compute the expectation of the AUC instead of the classification error. Next, we will solve the posterior densities for the parameters, \emph{i.e.,} $\prior^*(\thb_\classzero, \thb_\classone \mid \X, \y)$ in Eq. (\ref{eq:bee_auc_1}).

\subsection{Posterior Parameter Densities}
Following~\citep{dalton2011bayesian}, we assume normal-inverse-Wishart priors (with normal distribution for the mean and inverse-Wishart distribution for the covariances) of the form
\begin{align}
&\prior(\thb_\classzero, \thb_\classone)  = \prior(\Sigmab, \mub_\classzero, \mub_\classone) \nonumber \\ 
&= \prior(\Sigmab) \prior(\mub_\classzero \mid \Sigmab) \prior(\mub_\classone \mid \Sigmab) \nonumber \\
&\propto \det(\Sigmab)^{-(\kappa + \dimP + 1)/2} \exp \left(-\frac{1}{2}\Tr(\S\Sigmab^{-1})\right) \nonumber \\
&\times \det(\Sigmab)^{-1/2} \exp \left(-\frac{\nu_\classzero}{2}(\mub_\classzero - \mm_\classzero)^T\Sigmab^{-1}(\mub_\classzero - \mm_\classzero)\right) \nonumber \\
&\times \det(\Sigmab)^{-1/2} \exp \left(-\frac{\nu_\classone}{2}(\mub_\classone - \mm_\classone)^T\Sigmab^{-1}(\mub_\classone - \mm_\classone)\right). \nonumber
\end{align}
 The hyperparameters $\mm_{\classzero,\classone}\in\R^P$ and $\S\in\R^{P\times P}$
can be viewed as distribution centers for mean and covariance, respectively. Moreover,  $\nu_{\classzero, \classone}\in \R$, and $\kappa\in\R$ control the variability about these means. The intuitive meaning of these hyperparameters is that $\mm_{1,2}$ and $\S$ act as the most likely targets for the mean(s) and covariance of the distribution, while $\nu_{1,2}$ and $\kappa$ control how much the prior penalizes variability from $\mm_{1,2}$ and $\S$, respectively. In all our experiments, we set the parameters as $\mm_{\classzero,\classone}=\boldsymbol{0}$, $\S = {\bf I}$, $\nu_{1,2} = 0.5$ and $\kappa = P + 2$, as in~\citep{dalton2011bayesian}. 

\jussi{The normal-inverse-Wishart is the conjugate prior for this problem \cite{gelman2013bayesian}. The property that the posterior distribution follows the same parametric form as the prior distribution is called conjugacy; The conjugate family is mathematically convenient in that the posterior distribution follows a known parametric form \cite{gelman2013bayesian}.}
In order to obtain a closed-form solution and to ensure that prior distribution follows a normal inverse-Wishart distribution~\citep{gelman2013bayesian}, the values of the hyperparameters are required to be restricted. For more detailed discussion of the priors, we refer the reader to~\citep{dalton2011bayesian}.

Denote the sample means and covariances of class 1 and 2 as $\hat{\mub}_1, \hat{\mub}_2$ and $\hat{\Sigmab}_\classzero, \hat{\Sigmab}_\classone$, respectively. We can compute sample means and covariances from the data samples as $\hat{\mub}_j = \sum_i^{n_j} \x_i^j / n_j$ and $\hat{\Sigmab}_j = (1/(n_j - 1)) \sum_i^{n_j} (\x_i^j - \hat{\mub}_j) (\x_i^j - \hat{\mub}_j)^T$ for $j \in \{1,2\}$. Moreover, let $n_1$ and $n_2$ be the number of samples in the two classes.
Then, by substituting the prior into Eq.~(\ref{eq:posterior}), we obtain
\begin{align}
&\prior^*(\thb_\classzero, \thb_\classone) \propto \prior(\thb_\classzero, \thb_\classone) \nonumber \\
&\times \det(\Sigmab)^{-n_\classzero/2} \exp \bigg( -\frac{1}{2} \Tr
\left((n_\classzero - 1) \hat{\Sigmab}_\classzero \Sigmab^{-1}\right) \nonumber \\
&- {(n_\classzero/2)}(\mub_\classzero - \hat{\mub}_\classzero)^T \Sigmab^{-1}(\mub_\classzero - \hat{\mub}_\classzero) \bigg)  \nonumber \\
&\times \det(\Sigmab)^{-n_\classone/2} \exp\bigg( -\frac{1}{2} \Tr \left((n_\classone - 1) \hat{\Sigmab}_\classone\Sigmab^{-1}\right)  \nonumber  \\
&- {(n_\classone/2)}(\mub_\classone - \hat{\mub}_\classone)^T \Sigma^{-1}(\mub_\classone - \hat{\mub}_\classone) \bigg)  \nonumber \\
&\propto \det(\Sigmab)^{-(\kappa + n + \dimP + 1)/2} \nonumber \\  
&\times \exp\left( -\frac{1}{2} \Tr \left( \left( (n_\classzero - 1)\hat{\Sigmab}_\classzero +
(n_\classone - 1)\hat{\Sigmab}_\classone + S\right) \Sigmab^{-1}\right) \right) \nonumber \\
&\times \det(\Sigmab)^{-1/2} \exp\bigg( -\frac{1}{2}\Big( n_\classzero(\mub_\classzero - \hat{\mub}_\classzero)^T\Sigmab^{-1}(\mub_\classzero - \hat{\mub}_\classzero) \nonumber \\ 
&+ \nu_\classzero (\mub_\classzero - \mm_\classzero)^T \Sigmab^{-1} (\mub_\classzero - \mm_\classzero) \Big) \bigg) \nonumber \\
&\times \det(\Sigmab)^{-1/2} \exp\bigg( -\frac{1}{2}\Big( n_\classone(\mub_\classone - \hat{\mub}_\classone)^T\Sigmab^{-1}(\mub_\classone - \hat{\mub}_\classone) \nonumber \\
&+\nu_\classone (\mub_\classone - \mm_\classone)^T \Sigmab^{-1} (\mub_\classone - \mm_\classone) \Big) \bigg). \nonumber
\end{align}
\vspace*{-6pt}
By denoting the posterior hyperparameters as
\begin{align}
\kappa^* &= \kappa + n_\classzero + n_\classone = \kappa + n, \nonumber  \\ 
\nu_\classzero^* &= \nu_\classzero + n_\classzero, \, 
\nu_\classone^* = \nu_\classone + n_\classone,  \nonumber \\  
\S^* &= (n_\classzero - 1)\hat{\Sigmab}_\classzero + (n_\classone - 1)\hat{\Sigmab}_\classone + \S  \nonumber \\
& + \frac{n_\classzero \nu_\classzero}{n_\classzero + \nu_\classzero}(\hat{\mub}_\classzero - \mm_\classzero)(\hat{\mub}_\classzero - \mm_\classzero)^T \label{eq:KappaStar}\\
& + \frac{n_\classone \nu_\classone}{n_\classone + \nu_\classone}(\hat{\mub}_\classone - \mm_\classone)(\hat{\mub}_\classone - \mm_\classone)^T, \nonumber \\
\mm_\classzero^* &= \frac{\left( {n_\classzero \hat{\mub}_\classzero + \nu_\classzero \mm_\classzero} \right) }{ \left({n_\classzero + \nu_\classzero}\right)},
\mm_\classone^* = \frac{\left( {n_\classone \hat{\mub}_\classone + \nu_\classone \mm_\classone} \right) }{ \left({n_\classone + \nu_\classone}\right)},
\nonumber
\end{align}
we arrive at
\begin{equation}
\label{eq:posteriorParam}
\begin{aligned}
&\prior^*(\thb_\classzero, \thb_\classone) \propto \det(\Sigmab)^{-(\kappa^* + \dimP + 1)/2} \exp \left( -\frac{1}{2}\Tr\left(\S^*\Sigmab^{-1}\right)\right) \\
&\times \det(\Sigmab)^{-1/2} \exp \left(- \frac{\nu_\classzero^*}{2}(\mub_\classzero - \mm_\classzero^*)^T\Sigmab^{-1}(\mub_\classzero - \mm_\classzero^*)\right) \\
&\times \det(\Sigmab)^{-1/2} \exp \left(- \frac{\nu_\classone^*}{2}(\mub_\classone - \mm_\classone^*)^T\Sigmab^{-1}(\mub_\classone - \mm_\classone^*)\right) \\
&\propto \prior^*(\Sigmab) \prior^*(\mub_\classzero\mid \Sigmab) \prior^*(\mub_\classone\mid \Sigmab). \mbox{\qedhere}
\end{aligned}
\end{equation}

With the posterior parameter densities, we can now proceed to evaluate the expectation of the posterior of Eq. (\ref{eq:bee_auc_1}). \jussi{We note that the matrix $\S^*$ is positive-definite, since matrices $\hat{\Sigmab}_\classzero$ and $\hat{\Sigmab}_\classone$ are positive semi-definite and $\S = {\bf I}$. This ensures the correct definition of Inverse-Wishart distribution in the proof of Proposition 1. }

\subsection{Closed-Form Solution}
We can rewrite Eq.~(\ref{eq:bee_auc_1}) with the notation of Eq.~(\ref{eq:posteriorParam}) such that
\begin{equation}
\begin{aligned}
&\cbauc(\avar) = E[\text{AUC} \mid \avar] \\
&= \int \int_{\R^\dimP} \int_{\R^\dimP}\text{AUC}(\avar) \\
&\times 
\prior^*(\mub_\classzero|\Sigmab) \prior^*(\mub_\classone|\Sigmab) \prior^*(\Sigmab) \mathrm{d}\mub_\classzero \, \mathrm{d}\mub_\classone \mathrm{d}\Sigmab.
\end{aligned}
\label{eq:bee_auc_2}
\end{equation}
In order to evaluate the two inner integrals, we state the following lemma.

\begin{lemma}
\label{lemma:lemma1}
Assuming a fixed common covariance matrix $\Sigmab$ for the classes, the expected posterior $\emph{AUC}$ over class means $\mub_1$ and $\mub_2$ is given by
\begin{equation}
\label{eq:Lemma1}
\begin{aligned}
&\int_{\R^\dimP} \int_{\R^\dimP}\emph{AUC}(\avar) \prior^*(\mub_\classzero \mid \Sigmab) \prior^*(\mub_\classone \mid \Sigmab) \mathrm{d}\mub_\classzero \mathrm{d}\mub_\classone \\
&= \int_{\R^\dimP} \int_{\R^\dimP} \Phi \left( \frac{\avar^T ( \mub_\classone - \mub_\classzero )}{\sqrt{2 \avar^T \Sigmab \avar}} \right) \\
&\times f_{\{ \mm_\classzero^*, \Sigmab/\nu_\classzero^* \}}(\mub_\classzero) 
f_{\{ \mm_\classone^*, \Sigmab/\nu_\classone^* \}}(\mub_\classone) \mathrm{d}\mub_\classzero  \mathrm{d}\mub_\classone \\
&=  \Phi \left( \frac{\avar^T ( \mm_\classone^* - \mm_\classzero^* )}{\sqrt{2 \avar^T \Sigmab \avar}} \sqrt{\frac{2 \nu_\classzero^* \nu_\classone^*}{\nu_\classzero^* + \nu_\classone^* + 2 \nu_\classzero^*\nu_\classone^*}} \right).
\end{aligned}
\end{equation}
\end{lemma}


\begin{proof}
Let 
\begin{equation}
\label{eq:lemma1_leftHandSide}
\begin{aligned}
&M = \int_{\R^\dimP} \int_{\R^\dimP} \Phi \Big( \frac{\avar^T ( \mub_\classone - \mub_\classzero )}{\sqrt{2 \avar^T \Sigmab \avar}} \Big) \\
&\times f_{\{ \mm_\classzero^*, \Sigmab/\nu_\classzero^* \}}(\mub_\classzero) 
f_{\{ \mm_\classone^*, \Sigmab/\nu_\classone^* \}}(\mub_\classone) \mathrm{d}\mub_\classzero  \mathrm{d}\mub_\classone \\
\end{aligned}
\end{equation}
Given a fixed covariance matrix, the posterior density for the mean is Gaussian. Thus, we can write,
\begin{equation*}
\label{eq:lemma1_density_i}
\begin{aligned}
&f_{\{ \mm_\classc^*, \Sigmab/\nu_\classc^* \}}(\mub_\classc) = \frac{{\nu_\classc^*}^{\frac{\dimP}{2}}}{(2 \pi)^{\frac{\dimP}{2}} \det(\Sigmab)^{\frac{1}{2}} } \\
&\times \exp(-\frac{\nu_\classc^*}{2} (\mub_\classc - \mm_\classc^*)^T \Sigmab^{-1} (\mub_\classc - \mm_\classc^*) )
\end{aligned}
\end{equation*}
Replacing these expressions in Eq.(~\ref{eq:lemma1_leftHandSide}), we get
\begin{equation}
\label{eq:lemma1_leftHandSide_2}
\begin{aligned}
&M = \int_{\R^\dimP} \int_{\R^\dimP} \Phi \Big( \frac{\avar^T ( \mub_\classone - \mub_\classzero )}{\sqrt{2 \avar^T \Sigmab \avar}} \Big)\\ 
&\times 
\exp(-\frac{\nu_\classzero^*}{2} (\mub_\classzero - \mm_\classzero^*)^T \Sigmab^{-1} (\mub_\classzero - \mm_\classzero^*) )\\
&\times
\exp(-\frac{\nu_\classone^*}{2} (\mub_\classone - \mm_\classone^*)^T \Sigmab^{-1} (\mub_\classone - \mm_\classone^*) )\\
&\times 
\frac{{\nu_\classzero^*}^{\frac{\dimP}{2}}{\nu_\classone^*}^{\frac{\dimP}{2}}}{ (2 \pi)^\dimP \det(\Sigmab) } \mathrm{d}\mub_\classzero  \mathrm{d}\mub_\classone \\
\end{aligned}
\end{equation}

Since $\Sigmab$ is an invertible covariance matrix and by singular value decomposition, we can say, $\Sigmab = \Wb\Wb^T$ with $ \det(\Sigmab) = \det(\Wb)^2 $. Moreover, making changes of variables $ \zb_\classzero = \sqrt{\nu_\classzero^*}\Wb^{-1}(\mub_\classzero - \mm_\classzero^*)$ and $ \zb_\classone = \sqrt{\nu_\classone^*}\Wb^{-1}(\mub_\classone - \mm_\classone^*)$ we obtain,

\begin{equation}
\begin{aligned}
&M = \int_{\R^\dimP} \int_{\R^\dimP} \Phi \bigg( \frac{\avar^T ( \frac{1}{\sqrt{\nu_\classone^*}} \zb_\classone \Wb + \mm_\classone^* - \frac{1}{\sqrt{\nu_\classzero^*}} \zb_\classzero \Wb + \mm_\classzero^* )}{\sqrt{2 \avar^T \Sigmab \avar}} \bigg) \\
&\times \frac{1}{(2 \pi)^\dimP} \times \exp\Big( -\frac{\zb_\classzero^T \zb_\classzero + \zb_\classone^T \zb_\classone}{2} \Big) \mathrm{d}\zb_\classzero \mathrm{d}\zb_\classone \\
&=
\int_{\R^\dimP} \int_{\R^\dimP} \Phi \bigg( \frac{\frac{1}{\sqrt{\nu_\classone^*}} \avar^T \Wb \zb_\classone + \avar^T \mm_\classone^* - \frac{1}{\sqrt{\nu_\classzero^*}} \avar^T \Wb \zb_\classzero - \avar^T \mm_\classzero^*}{\sqrt{2 \avar^T \Sigmab \avar}} \bigg) \\
&\times \frac{1}{(2 \pi)^\dimP} \times \exp\Big( -\frac{\zb_\classzero^T \zb_\classzero + \zb_\classone^T \zb_\classone}{2} \Big) \mathrm{d}\zb_\classzero \mathrm{d}\zb_\classone.
\end{aligned}
\end{equation}

Let us define
\begin{equation}
\begin{aligned}
&\bar{\avar_\classzero} = \frac{\Wb^T \avar}{\sqrt{\nu_\classzero^*}\sqrt{2 \avar^T \Sigmab \avar}},  
&\bar{\avar_\classone} = \frac{\Wb^T \avar}{\sqrt{\nu_\classone^*}\sqrt{2 \avar^T \Sigmab \avar}}\\
&\bar{\bvar_\classzero} = \frac{\avar^T \mm_\classzero^*}{\sqrt{2 \avar^T \Sigmab \avar}}, 
&\bar{\bvar_\classone} = \frac{\avar^T \mm_\classone^*}{\sqrt{2 \avar^T \Sigmab \avar}}\\
&\|\bar{\avar_\classzero}\|^2 = \frac{1}{2 \nu_\classzero^*}, 
&\|\bar{\avar_\classone}\|^2 = \frac{1}{2 \nu_\classone^*}
\end{aligned}
\end{equation}
to obtain

\begin{equation}
\label{eq:M}
\begin{aligned}
&M = \int_{\R^\dimP} \int_{\R^\dimP} \Phi([\bar{\avar_\classone};-\bar{\avar_\classzero}]^T [\zb_\classone; \zb_\classzero] + (\bar{\bvar_\classone} - \bar{\bvar_\classzero})) \\
&\times \frac{1}{(2 \pi)^\dimP} \times \exp\Big( -\frac{[\zb_\classzero; \zb_\classone]^T [\zb_\classzero; \zb_\classone]}{2} \Big) \mathrm{d}\zb_\classzero \mathrm{d}\zb_\classone \\
&= \int_{\R^\dimP} \int_{\R^\dimP} \int_{R(x)} \frac{1}{(2 \pi)^{\dimP + \frac{1}{2}}} \\ 
&\times \exp \bigg( - \frac{x^2 + [\zb_\classzero;\zb_\classone]^T [\zb_\classzero;\zb_\classone]}{2}\bigg) \mathrm{d}x \mathrm{d}\zb_\classzero \mathrm{d}\zb_\classone,
\end{aligned}
\end{equation}
where $R(x) = \{ x | x < [\bar{\avar_\classone}; -\bar{\avar_\classzero}]^T [\zb_\classone;\zb_\classzero] + (\bar{\bvar_\classone}-\bar{\bvar_\classzero})\}$. Analogously to \cite[Appendix B]{dalton2011bayesian} we conclude that Eq. (\ref{eq:M}) gives the error of the $2\dimP + 1$ dimensional $2$-class classifier $\bar{g}([\zb_\classzero;\zb_\classone],x) = [\bar{\avar_\classone}; -\bar{\avar_\classzero}]^T [\zb_\classone;\zb_\classzero] - x + (\bar{\bvar_\classone}-\bar{\bvar_\classzero})$ originating from the $\classzero$ Gaussian class with zero mean and identity covariance. Thus, from \cite[Eq.~9]{dalton2011bayesian} it follows
\begin{equation*}
\begin{aligned}
&M = \Phi \left (\frac{\bar{\bvar_\classone} - \bar{\bvar_\classzero}}{\sqrt{\|[\bar{\avar_\classone};-\bar{\avar_\classzero}]\|^2 + 1}} \right) \\
&= \Phi \left( \frac{\avar^T(\mm_\classone^* - \mm_\classzero^*)}{\sqrt{2\avar^T\Sigmab \avar} \sqrt{\frac{1}{2\nu_\classone^*} + \frac{1}{2\nu_\classzero^*} + 1}} \right)\\
&= \Phi \left( \frac{\avar^T(\mm_\classone^* - \mm_\classzero^*)}{\sqrt{2\avar^T\Sigmab \avar}} \sqrt{\frac{2\nu_\classzero^*\nu_\classone^*}{\nu_\classzero^* + \nu_\classone^* + 2\nu_\classzero^*\nu_\classone^*}} \right).
\end{aligned}
\end{equation*}
Thus we have concluded the proof. 
\end{proof}

Substitution of the above result into Eq.~(\ref{eq:bee_auc_2}) yields
\begin{equation}
\begin{aligned}
&\cbauc(\avar) \\
& = \int \Phi \left( \frac{\avar^T ( \mm_\classone^* - \mm_\classzero^* )}{\sqrt{2 \avar^T \Sigmab \avar}} \sqrt{\frac{2 \nu_\classzero^* \nu_\classone^*}{\nu_\classzero^* + \nu_\classone^* + 2 \nu_\classzero^*\nu_\classone^*}} \right) \prior^*(\Sigmab) \mathrm{d}\Sigmab.
\end{aligned}
\label{eq:bee_auc_3}
\end{equation}

Next, we will state the main result of the paper solving this integral.

\begin{proposition}
Let $\psi(\cdot): \R^\dimP \mapsto \{\classzero, \classone\}$ denote a binary linear classifier parameterized by coefficients
$\avar\in\R^\dimP$ and intercept $b\in\R$. The Bayesian Area Under Curve estimate given training samples
${\X} = \{\x_1, \x_2,\ldots, \x_N\}$ assuming a common inverse-Wishart distributed covariance for both classes is then given as 
\begin{equation}
\begin{aligned}
&\normalfont\cbauc(\avar) = \\
& = \frac{1}{2} + \frac{\sgn({A^*})}{2} I \left(\frac{{A^*}^2}{{A^*}^2 + \avar^T\S^*\avar};\frac{1}{2},\frac{\kappa^* - \dimP + 1}{2}\right).
\end{aligned}
\end{equation}
with $I(x; a,b)$ the regularized incomplete beta function and
\[
A^* = \frac{\avar^T(\mm_\classone^* - \mm_\classzero^*) \sqrt{{\nu_\classzero^*\nu_\classone^*}}} {\sqrt{\nu_\classzero^* + \nu_\classone^* + 2 \nu_\classzero^* \nu_\classone^*}},
\]
and $\nu_1^*$, $\nu_2^*$, $\mm_1^*$, $\mm_2^*$, $\S^*$ and $\kappa^*$ as in Eq.~\ref{eq:KappaStar}.
\label{prop:proposition2}
\end{proposition}


\begin{proof}
The posterior density $\prior^*(\Sigmab)$ has an inverse Wishart distribution, \textit{i.e.},
\begin{equation*}
\begin{aligned}
\prior^*(\Sigmab) = \text{InverseWishart}(\Sigmab;\S^*,\kappa^*)
\end{aligned}
\end{equation*}
with hyperparameters $\S^*$ and $\kappa^*$. If we take
\[
A^* = \frac{\avar^T(\mm_\classone^* - \mm_\classzero^*) \sqrt{{\nu_\classzero^*\nu_\classone^*}}} {\sqrt{\nu_\classzero^* + \nu_\classzero^* + 2 \nu_\classzero^* \nu_\classone^*}},
\]
then Eq.~(8) gets the form
\begin{equation}
\label{eq:bee_auc_5}
\cbauc(\avar)
=\int_{R(\Sigmab)} \Phi\left(\frac{A^*}{\sqrt{\avar^T\Sigmab \avar}}\right) \prior^*(\Sigmab) \, \mathrm{d}\Sigmab, 
\end{equation}
where $R(\Sigmab) = \{ \Sigmab \in \R^{\dimP \times \dimP} \mid \Sigmab \text{ positive definite}\}$. By applying~\citep[Lemma~E.1]{dalton2011bayesian}, we arrive at
\begin{equation}
\begin{aligned}
&\int_{R(\Sigmab)} \Phi\left(\frac{A^*}{\sqrt{\avar^T\Sigmab \avar}}\right) \prior^*(\Sigmab) \mathrm{d}\Sigmab  \\
& = \frac{1}{2} + \frac{\sgn({A^*})}{2} I \left(\frac{{A^*}^2}{{A^*}^2 + \avar^T\S^*\avar};\frac{1}{2},\frac{\kappa^* - \dimP + 1}{2}\right).
\end{aligned}
\label{prop:proposition2_sol}
\end{equation}
This concludes the proof. 
\end{proof}


A similar solution can also be derived for scaled identity covariances of the form $\Sigmab = \sigma {\bf I}$; see the supplementary material.

\section {EXPERIMENTAL RESULTS}
\label{sec:experiments}

In this section, we study the accuracy of the proposed estimator using synthetic and real-world data.
We consider three experimental cases with the following settings.

\textbf{\em Synthetic Data}---The first experiment involves a synthetic $\dimP$-dimensional  multivariate Gaussian distributed dataset. Samples for two classes are generated randomly from $\dimP$-dimensional normal probability distribution with a common covariance matrix $\Sigmab$ and means $\mub_1$ and $\mub_2$. We choose $\mub_\classzero = \mathbf{0}$ and $\mub_\classone = \mathbf{1}$ of dimension $\dimP = 4, 10 \text{ and } 100$ and a common covariance matrix $\Sigmab = \mathbf{I}$, where $\mathbf{I}$ is an identity matrix of $\dimP \times \dimP$ dimension. 

\textbf{\em High Dimensionality: $P\gg N$}---For the next experiment, we use the preprocessed ovarian cancer data set from the FDA-NCI Clinical Proteomics Program Databank. The dataset contains $216$ observations with $4000$ features, where the observations are divided into two classes: $121$ patients are in cancer disease group and $95$ patients are in healthy group. The task is to distinguish between cancer versus healthy patients from this dataset. \jussi{We note that we assume that the true covariance $\Sigmab$ is full-rank, but the sample covariance matrices can be singular as they need not to be inverted and the matrix $\S^*$ is guaranteed to be non-singular.}

\textbf{\em Moderate Dimensionality: $P\approx N$}---The third experiment uses magnetoencephalography (MEG) mind reading challenge dataset of ICANN 2011 conference\footnotemark. The task was to predict the type of video shown to the test subject based on the MEG brain measurements. Although this was a 5-class problem, the organizers considered a binary task: separating videos with a plot from those without one. The approach winning the competition~\citep{huttunen2012MVAP} extracted features with 408 dimensions. We use these features in this work.  
\footnotetext{\url{http://www.cis.hut.fi/icann2011/mindreading.php}}

\subsection{Estimator Accuracy Evaluation}
\label{subsec:acc_assessment}

\begin{figure*}[htb]
\begin{center}
\subfigure[]{\includegraphics[width=0.25\textwidth]{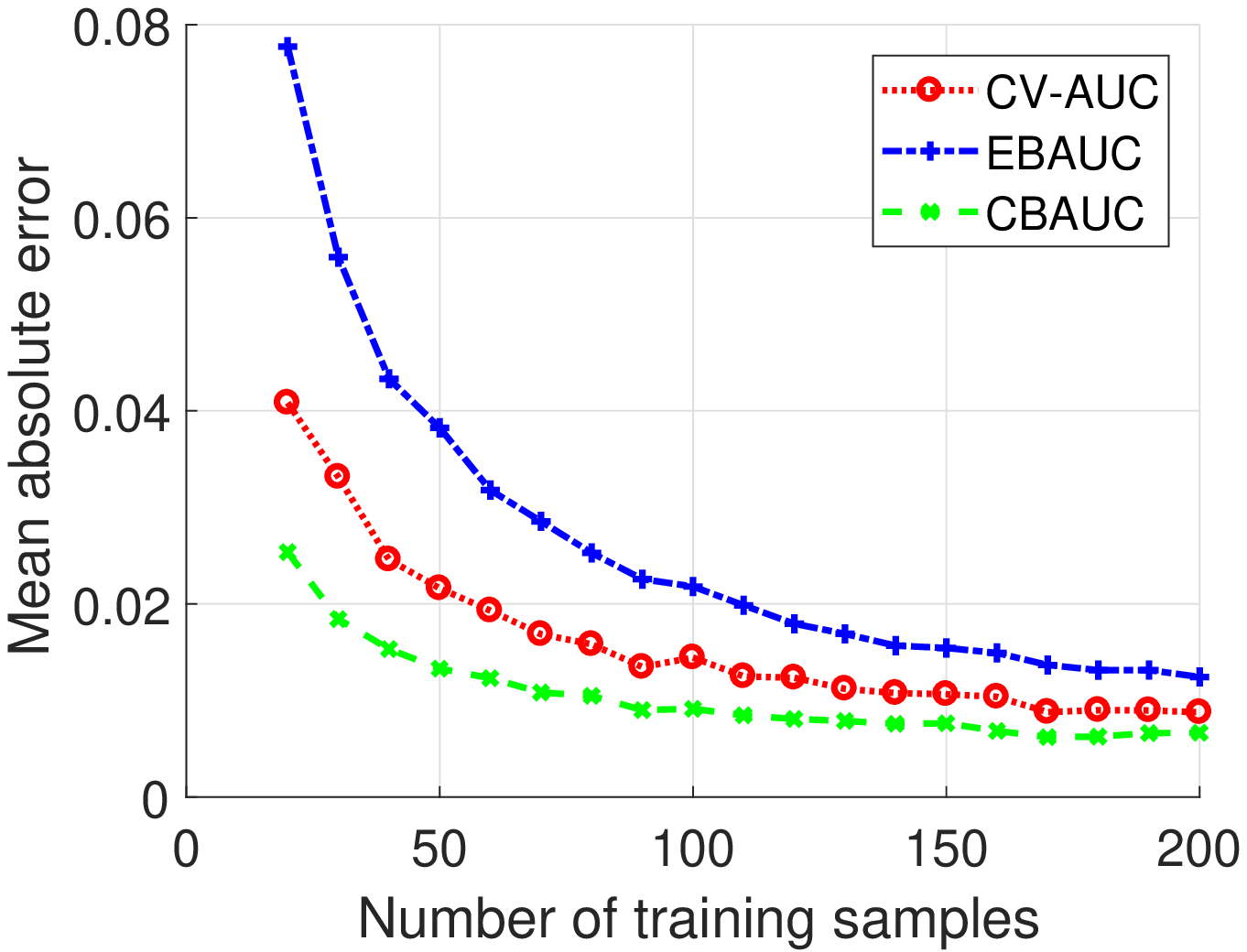}\label{fig:acc_diff_synthetic_case_4D}
}
\subfigure[]{\includegraphics[width=0.25\textwidth]{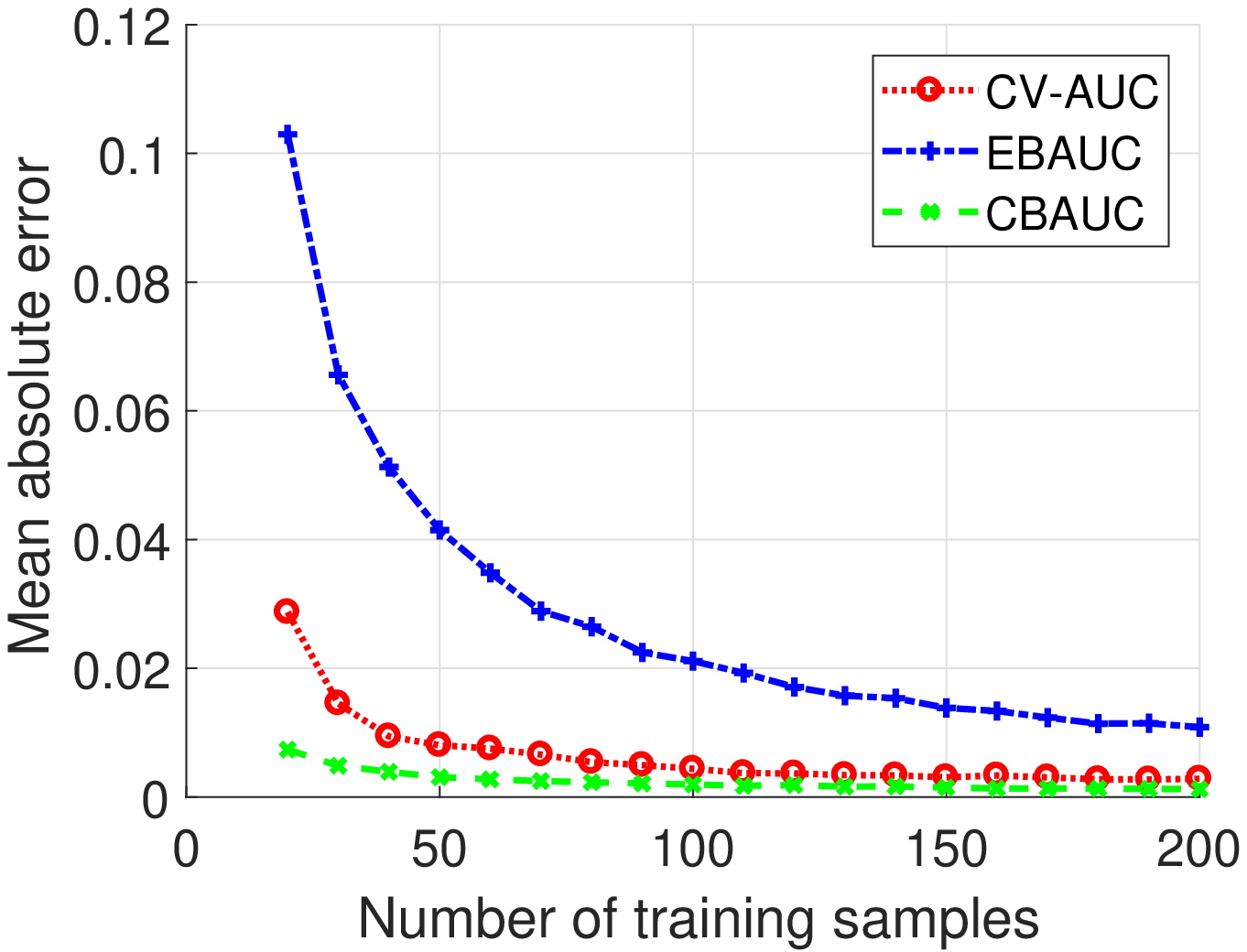}\label{fig:acc_diff_synthetic_case_10D}}
\subfigure[]{\includegraphics[width=0.25\textwidth]{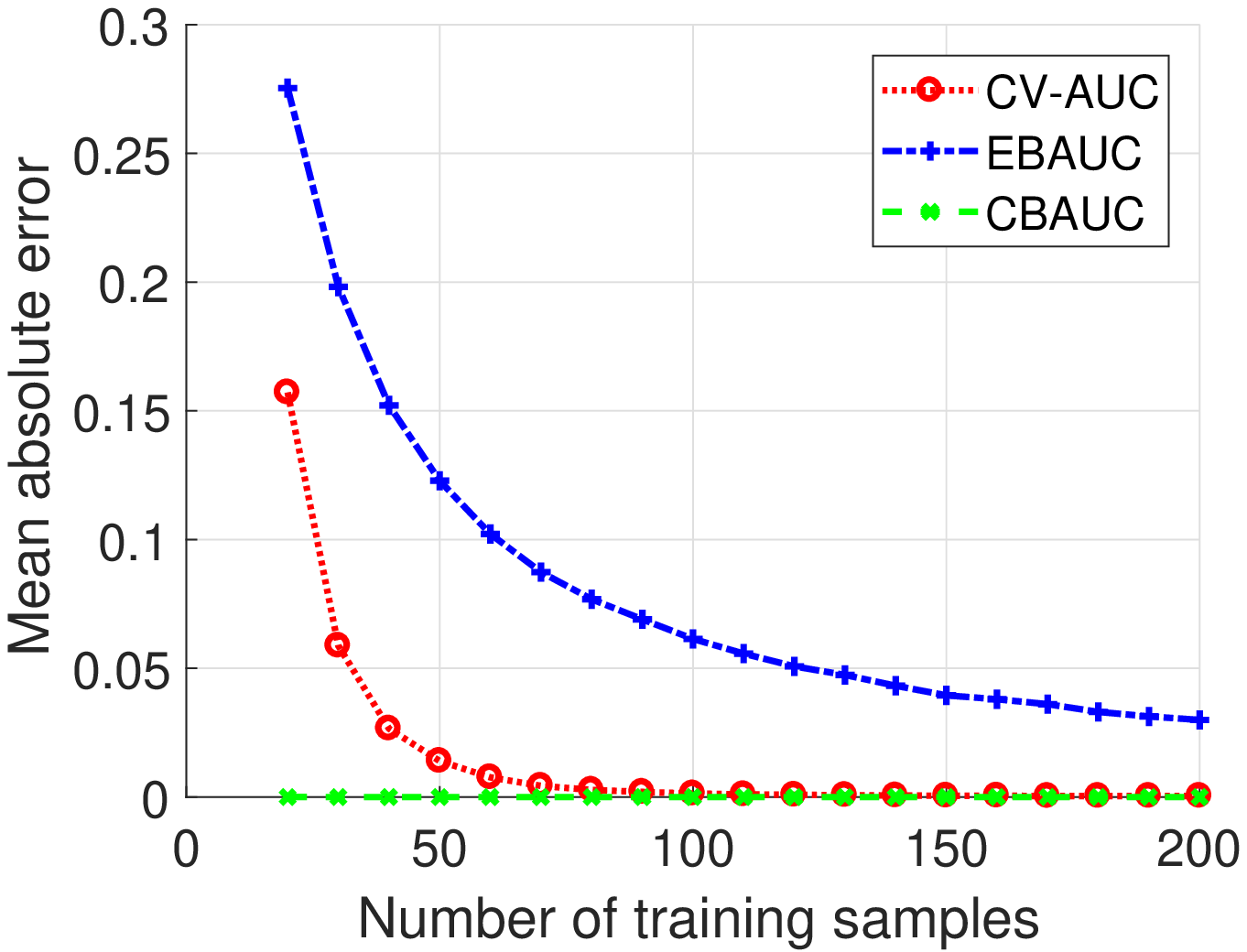}\label{fig:acc_diff_synthetic_case_100D}} \\

\subfigure[]{\includegraphics[width=0.25\textwidth]{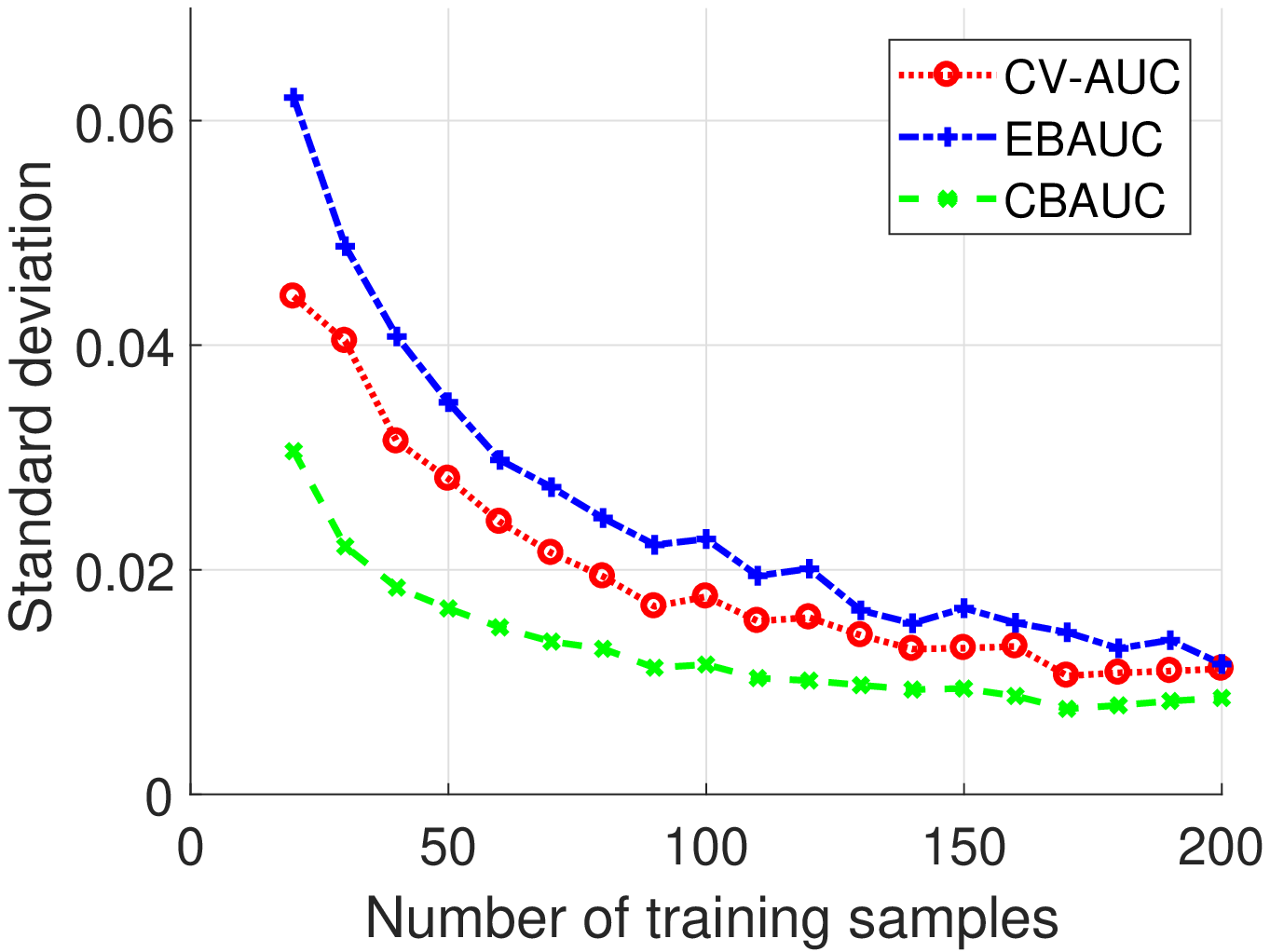}\label{fig:acc_std_4D}}
\subfigure[]{\includegraphics[width=0.25\textwidth]{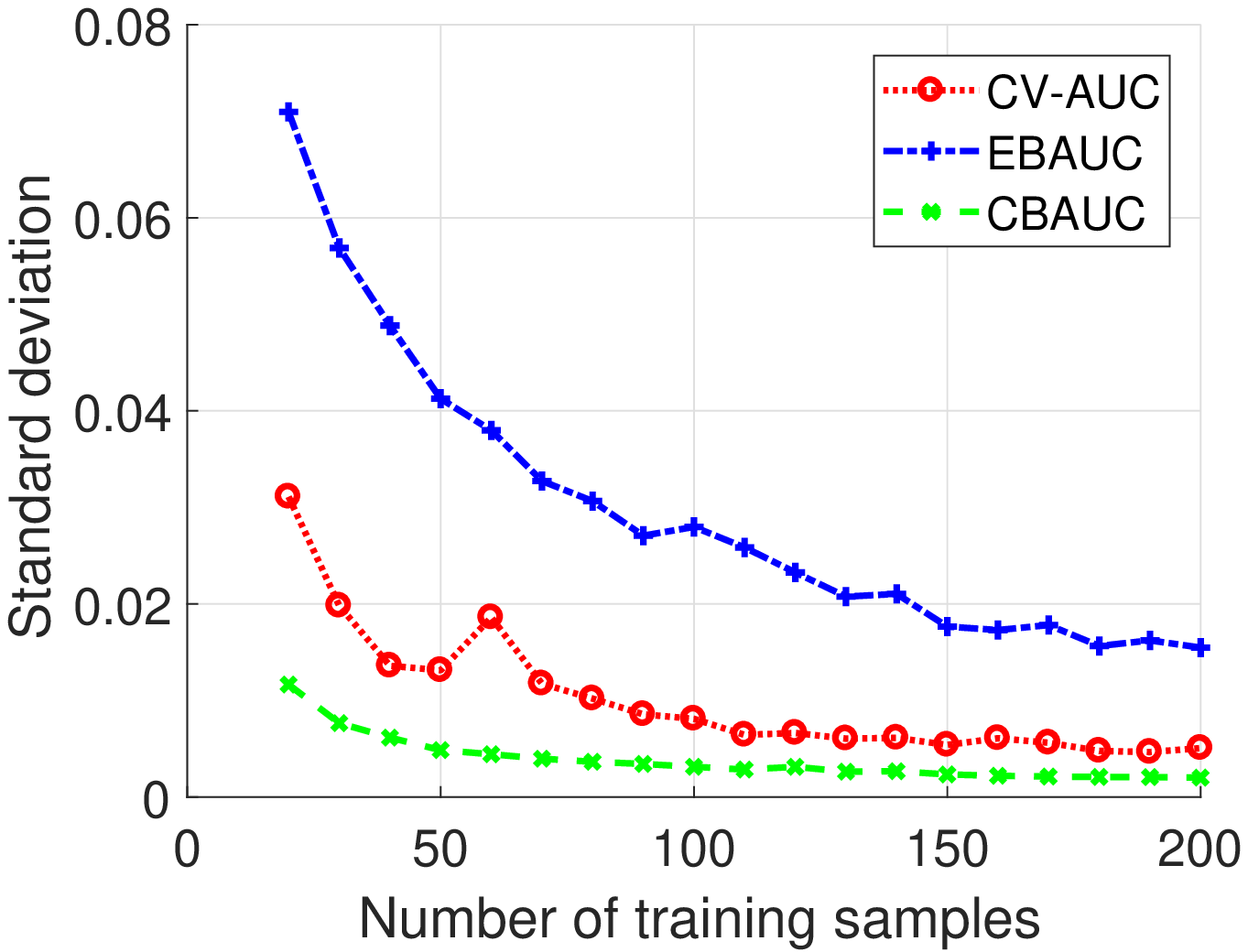}\label{fig:acc_std_10D}}
\subfigure[]{\includegraphics[width=0.25\textwidth]{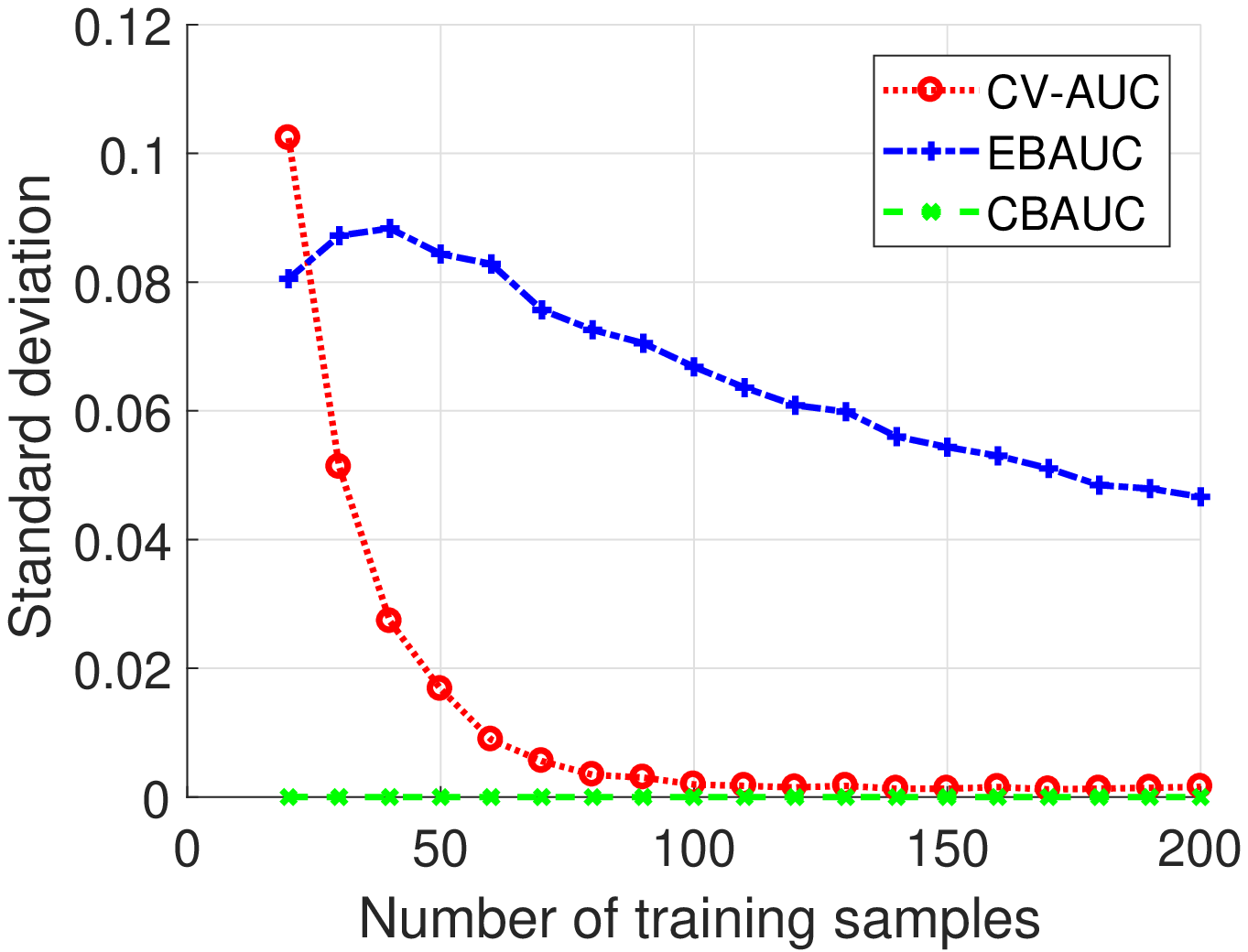}\label{fig:acc_std_100D}}

\caption{Accuracy assessment using multivariate Gaussian data. Mean absolute error (MAE) of three AUC estimators to the true AUC with dimensionality (a) P = 4, (b) P = 10, and (c) P = 100. The AUC estimators are 5-fold cross-validation AUC (CV-AUC) (red curve), empirical Bayesian AUC (EBAUC) (blue curve) and proposed closed-form Bayesian AUC (CBAUC) (green curve). Panels (d) P = 4, (e) P = 10, and (f) P = 100 show the standard deviations of the erros of AUC estimates.}
\label{fig:synthetic_acc}
\end{center}
\end{figure*}

\begin{figure}[htb]
\begin{center}
\subfigure[]{\includegraphics[width=4cm]{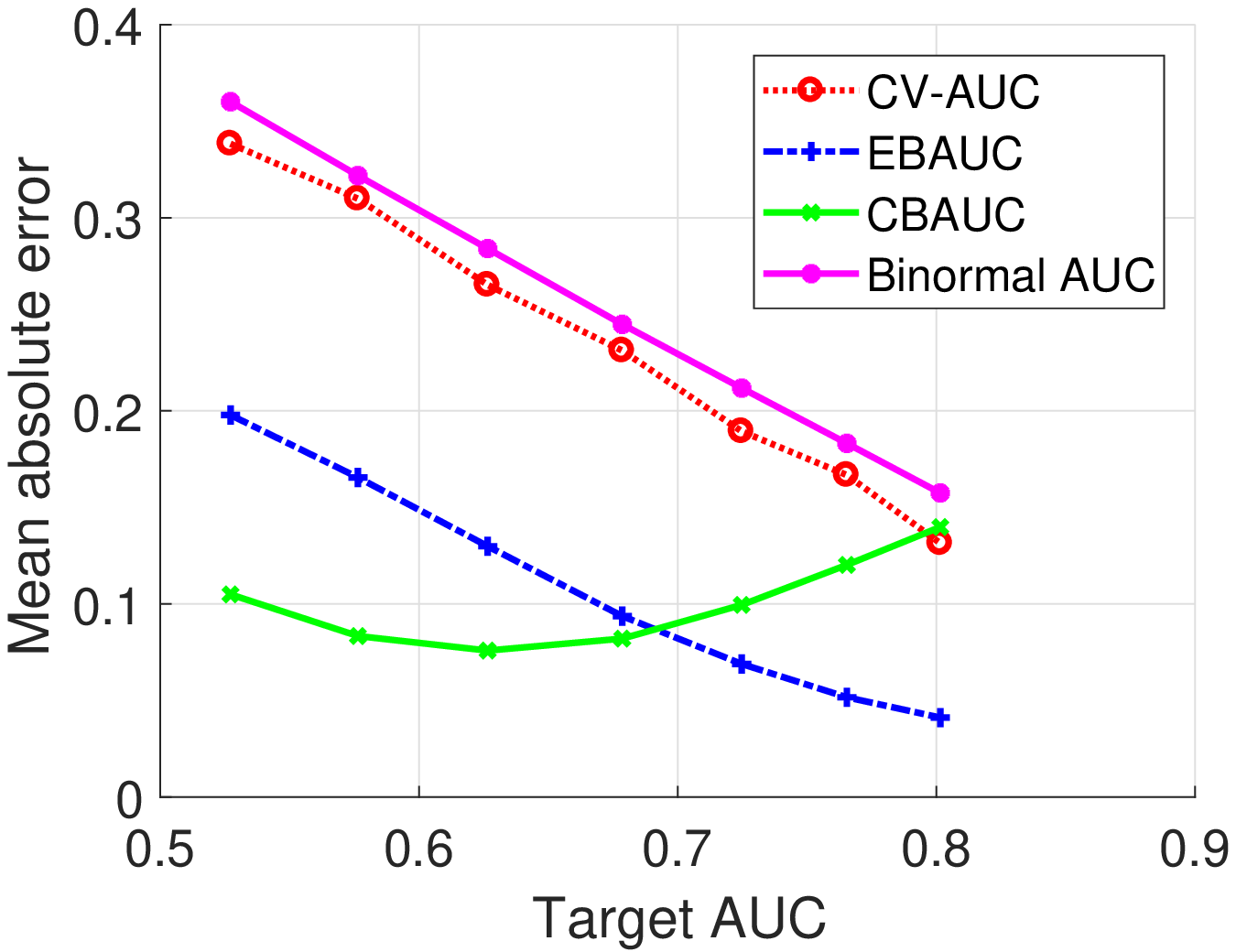}}
\subfigure[]{\includegraphics[width=4cm]{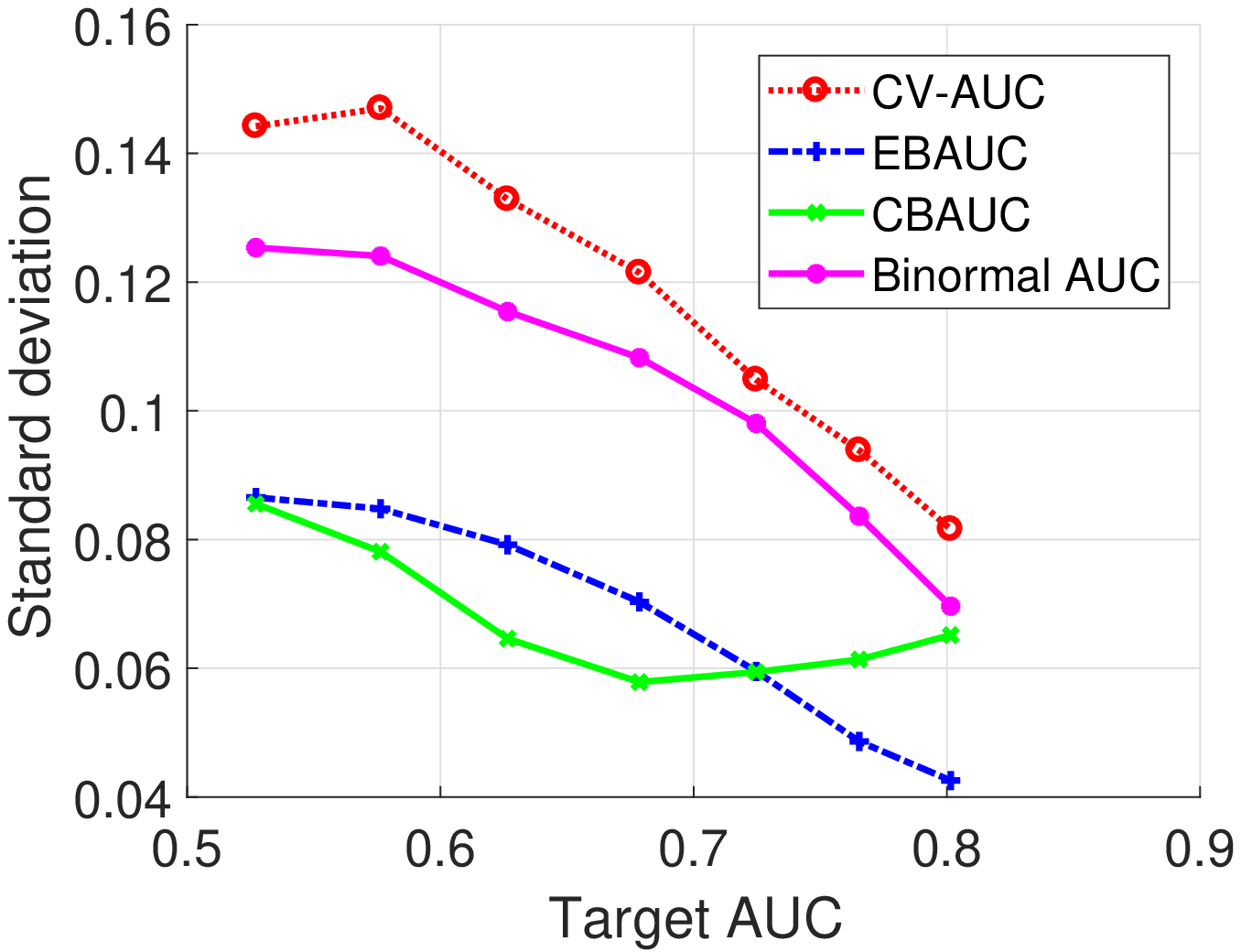}} 
\caption{Accuracy assessment with respect to the target AUC. (a) The mean absolute error of different estimators with respect to the target AUC. (b) The standard deviation of errors of different estimators with respect to the true AUC. The AUC estimators are 5-fold cross-validation AUC (CV-AUC) (red curve), empirical Bayesian AUC (EBAUC) (blue curve), proposed closed-form Bayesian AUC (CBAUC) (green curve), and sample-binormal AUC (pink curve). }
\label{fig:target_auc}
\end{center}
\end{figure}

\begin{figure}[htb]
\begin{center}
{\includegraphics[width=4cm]{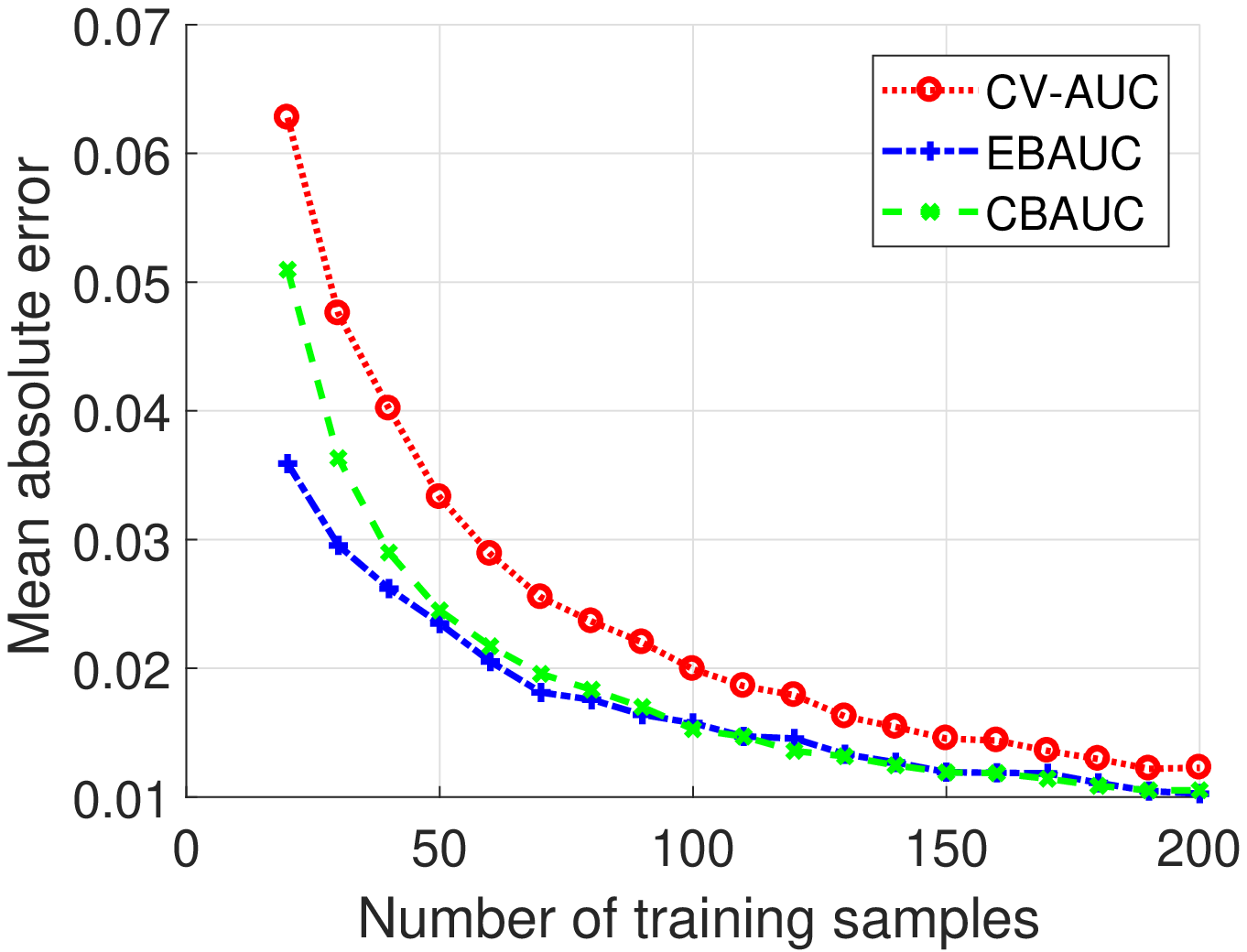}
}
\caption{Accuracy assessment using multivariate Gaussian data when {$\Sigma_\classzero \neq \Sigma_\classone$}. Mean absolute error (MAE) of three AUC estimators to the true AUC. The error estimators are 5-fold cross-validation AUC (CV-AUC) (red curve), empirical Bayesian AUC (EBAUC) (blue curve) and proposed closed-form Bayesian AUC (CBAUC) (green curve). }
\label{fig:synthetic_acc_with_diff_cov}
\end{center}
\end{figure}

\begin{figure*}[htb]
\begin{center}
\subfigure[]{\includegraphics[width=0.3\textwidth]{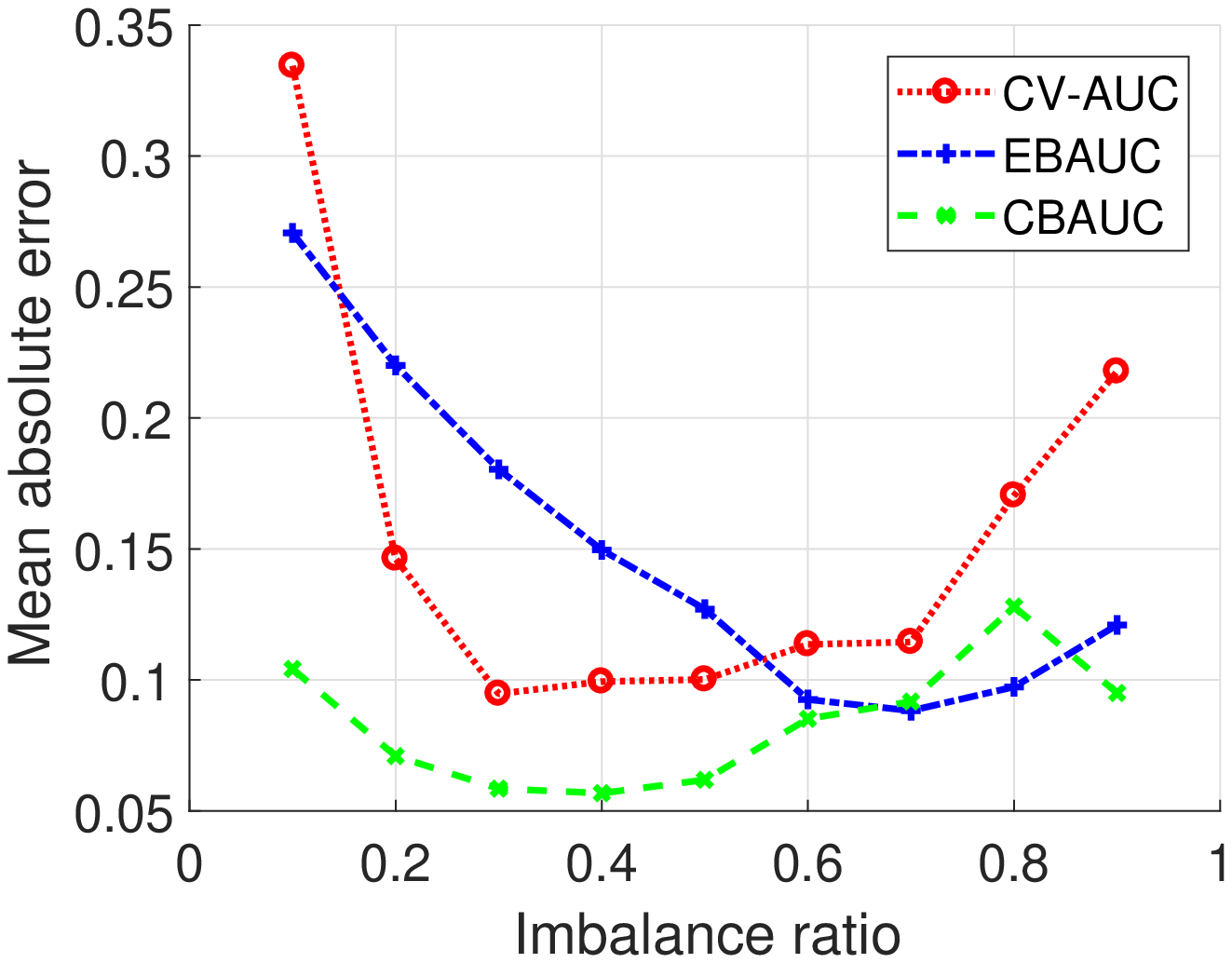}\label{fig:im_diff_synthetic_case_2D_N_10}
}
\subfigure[]{\includegraphics[width=0.3\textwidth]{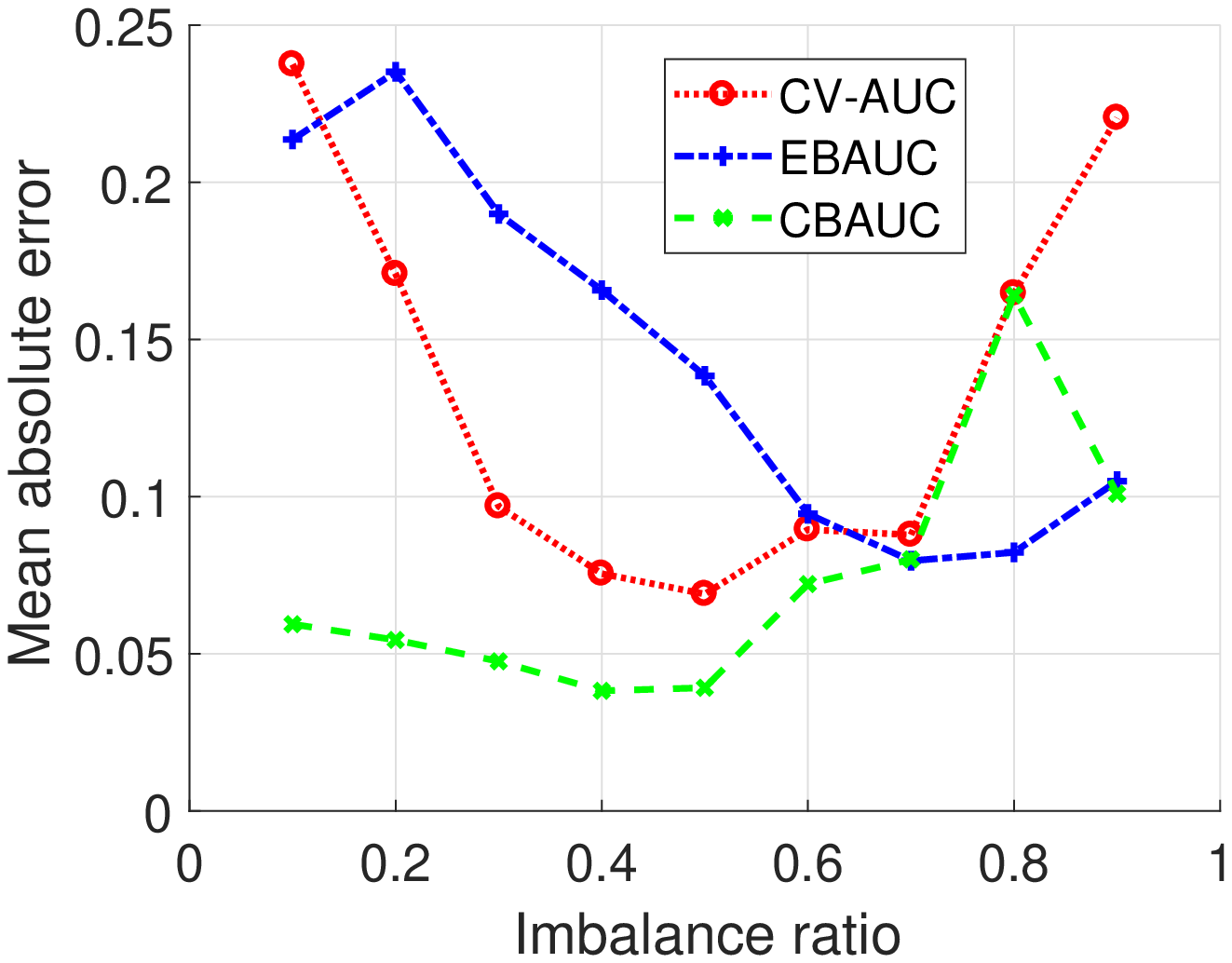}\label{fig:im_diff_synthetic_case_4D_N_10}
}
\subfigure[]{\includegraphics[width=0.3\textwidth]{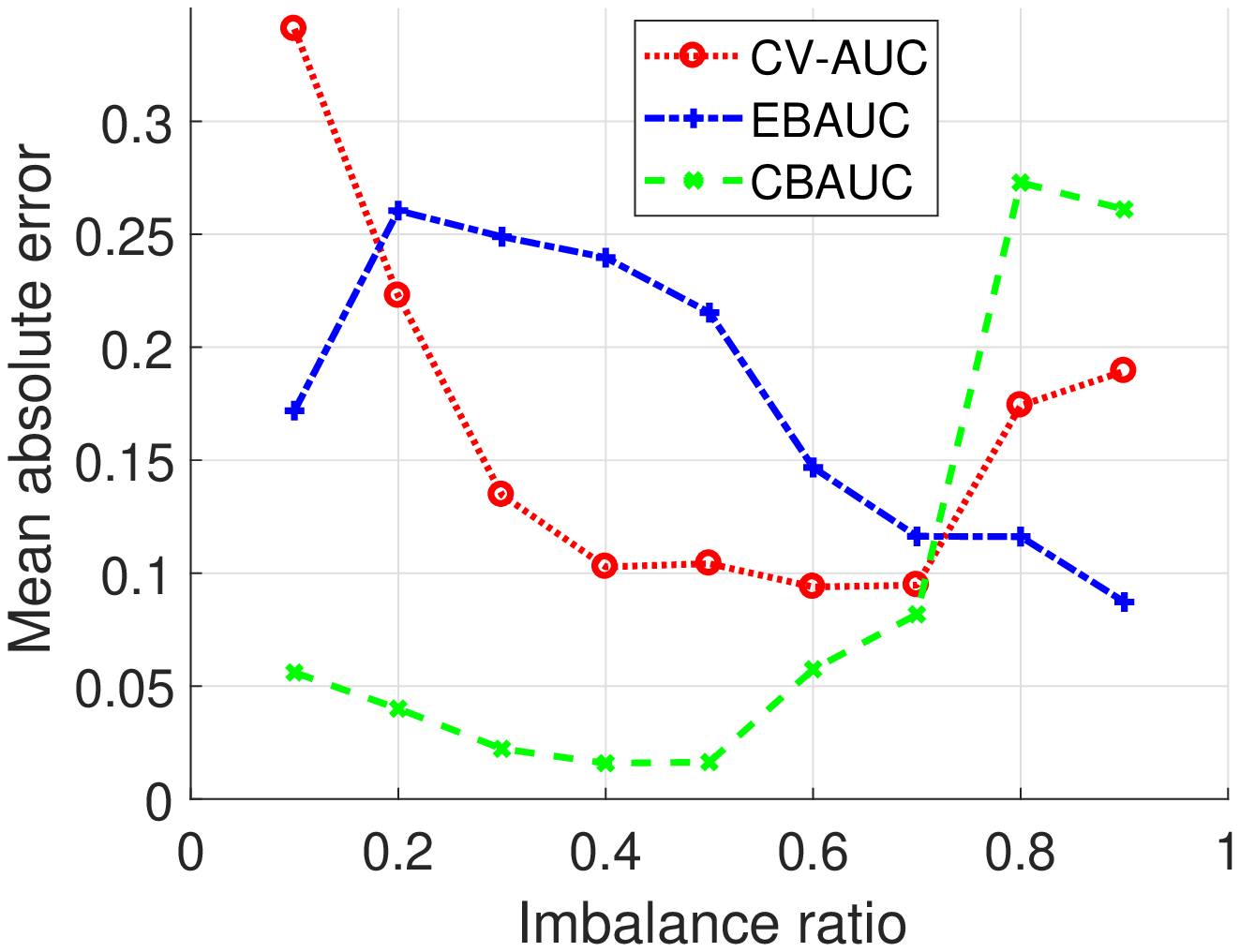}\label{fig:im_diff_synthetic_case_10D_N_10}
}
\\
\subfigure[]{\includegraphics[width=0.3\textwidth]{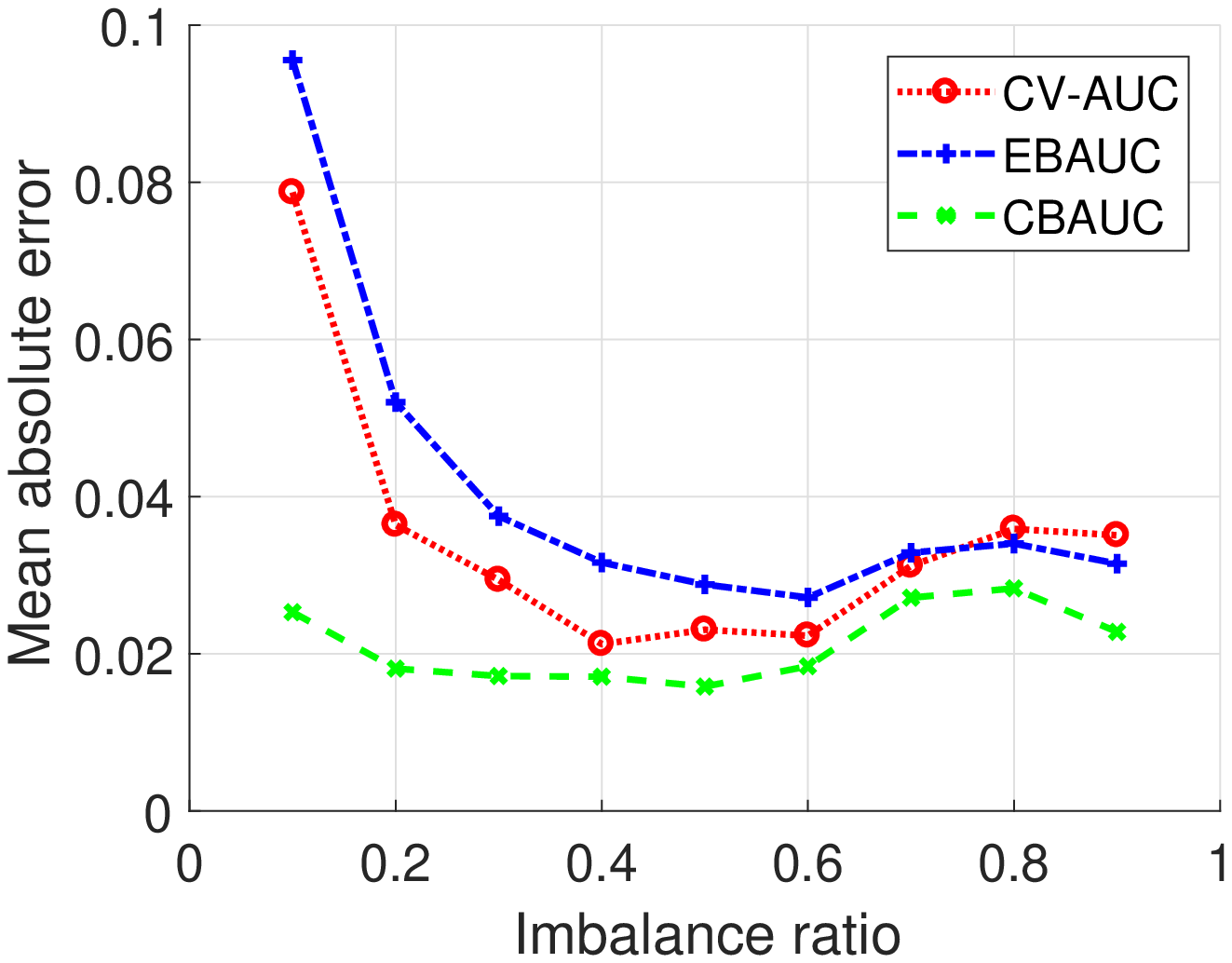}\label{fig:im_diff_synthetic_case_2D_N_100}
}
\subfigure[]{\includegraphics[width=0.3\textwidth]{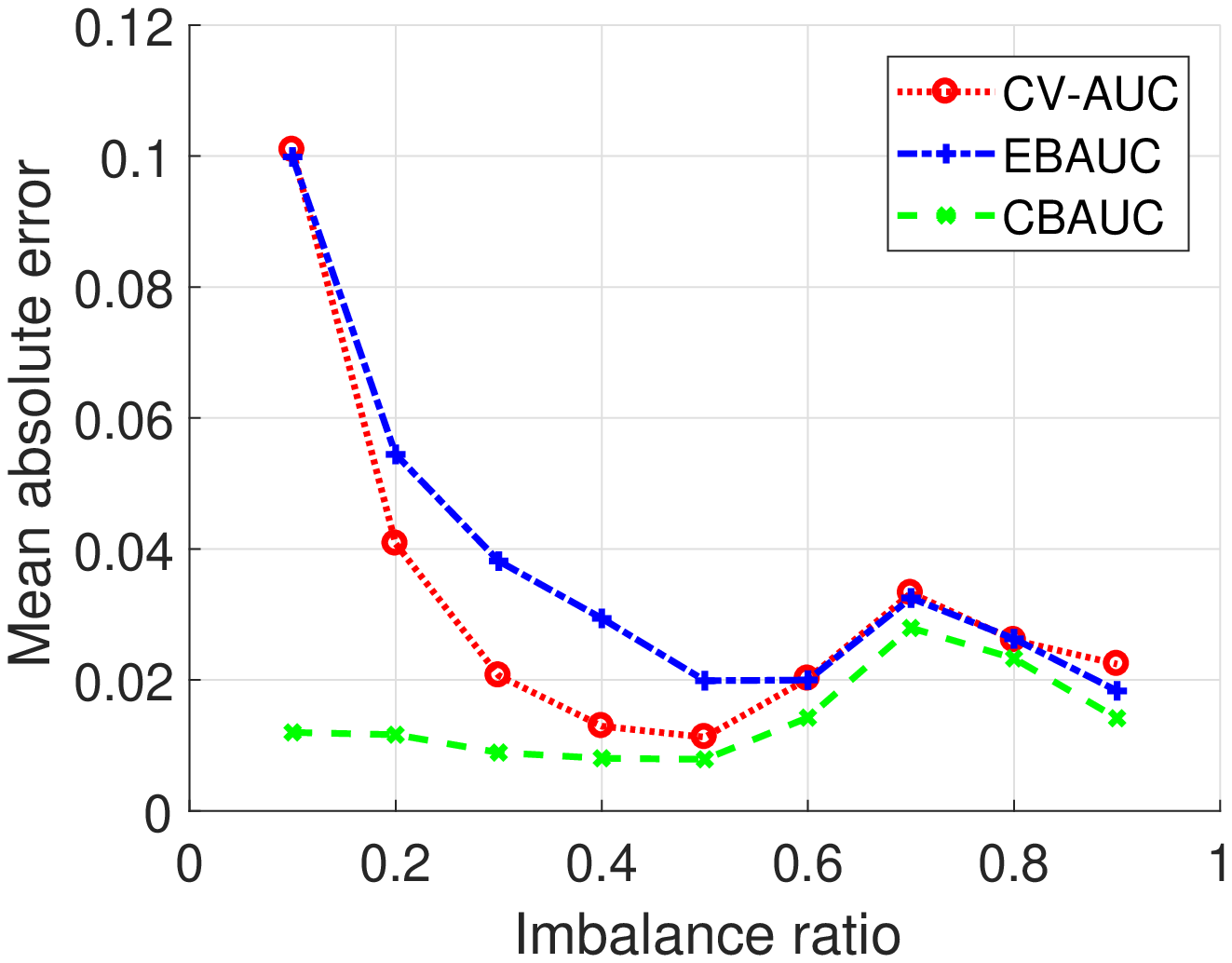}\label{fig:im_diff_synthetic_case_4D_N_100}
}
\subfigure[]{\includegraphics[width=0.3\textwidth]{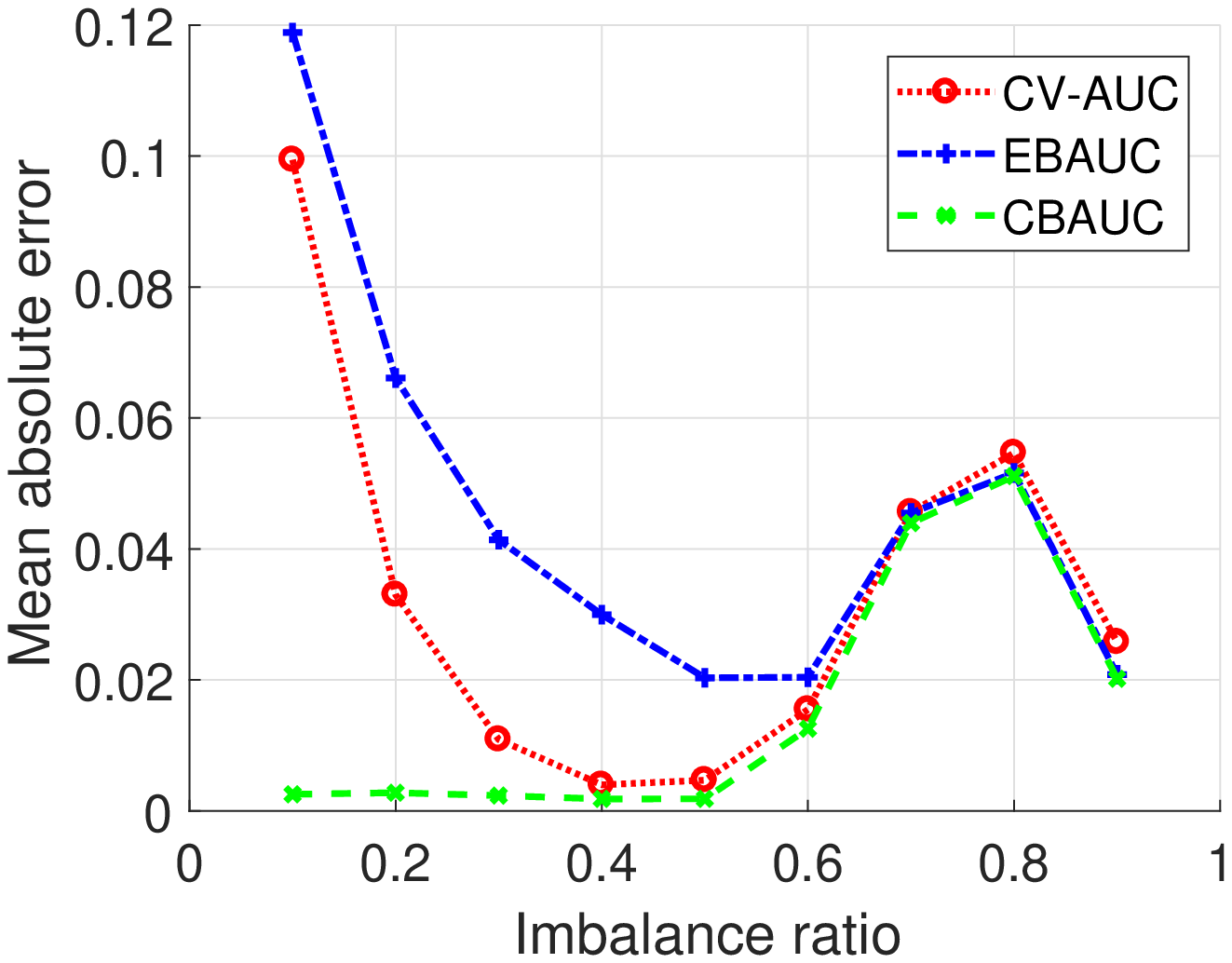}\label{fig:im_diff_synthetic_case_10D_N_100}
}
\caption{Accuracy assessment with imbalanced class in multivariate Gaussian data. Mean absolute error (MAE) of three AUC estimators to the true AUC with dimensionality of P = 2 ((a) and (d)), P = 4 ((b) and (e)), and P = 10 ((c) and (f)) and N = 10 (top row) and N = 100 (bottom row). The error estimators are 5-fold cross-validation AUC (CV-AUC) (red curve), empirical Bayesian AUC (EBAUC) (blue curve) and proposed closed-form Bayesian AUC (CBAUC) (green curve).}
\label{fig:imbalanced_case}
\end{center}
\end{figure*}

We assess the accuracy of our proposed estimator and compare it against two alternative estimators: the 5-fold cross-validated AUC estimator and the Empirical Bayesian AUC estimator. The former approach splits the training data into $K=5$ folds, and uses each fold in turn as a validation set. The resulting five AUC estimates are then averaged to produce the final estimate. The latter approach uses the approach proposed by~\citep{dalton2015ROC}, where the detection threshold slides in the ROC space, and the false and true positive rates are estimated using the Bayesian error estimation approach.

The reference AUC metric is different for the simulated and real-world cases. For simulated data, we know the parameters of the distribution where the sample is drawn and the true AUC (TAUC) and use that as the ground truth. We then compare the estimates given by the three estimators with the true AUC. \jussi{The true AUC is the true AUC conditioned on the classifier (see Eq. (2)). It depends on the training sample which determines the classifier parameters ~\citep{raudys1980dimensionality}. }

For real-world data, we split the data into training and testing---either at random (ovarian cancer dataset) or using a predefined train-test split, if such a commonly used split is defined (MEG dataset). We then use the hold-out test set AUC as the ground truth. 

In all cases, we are interested in the behavior of the estimators as the training sample size varies. With the simulated data, we can obviously generate training sets of any size, but with the real world cases, we create a training set by randomly subsampling  10\%, 15\%, 20\%, ..., 90\%, 95\%  of the original data. Finally, all experiments are iterated $1000$ times in order to average out any random effects.


The proposed method can estimate the accuracy of any linear classifier: linear discriminant, support vector machine or logistic regression. In the experiments, we use the $\L_2$ regularized logistic regression trained by {LIBLINEAR} toolbox~\citep{fan2008liblinear}.

\subsubsection{Multivariate Gaussian Data}
\label{subsec:mvdata_accuracy}

We construct a two-class multivariate Gaussian dataset by drawing $N = 2n$ samples where $n = 10, 15, \ldots, 100$; with $n$ samples from both classes.
Fig.~\ref{fig:synthetic_acc} (a)-(c) illustrates the mean absolute error (MAE) of the three AUC estimators with respect to the known population level ground truth calculated from the true distribution parameters. Table~\ref{tab:trueauc} \jussi{lists the true AUC calculated from the true distribution parameters, averaged over $1000$ classifiers based on samples drawn.}
\heikki{Additionally, the table also shows the AUC based on Binormal model based on population parameters~\citep{chakraborty2017observer,pepe2003statistical}, which is the AUC for the Bayes classifier \jussi{in the Gaussian case}. The Binormal-Bayes AUC is always the largest, because the classifier is based on population statistics instead of sample statistics.}

\begin{table}[!t]
\caption{\label{tab:trueauc}The true AUC values for the synthetic data experiment for each dimensionality ($P = 4, 10$, and 100) calculated from the true distributions of the parameters and averaged over the classifiers trained.}
\centering
\begin{tabular}{ l || c c c }
  \emph{Number of training samples} & $P = 4$ & $P = 10$ & $P = 100$ \\
  \hline\hline
  20 &   0.9528 & 0.9901 & 1.0000 \\
  50 &   0.9653 & 0.9947 & 1.0000 \\
  100 &  0.9688 & 0.9957  & 1.0000 \\
  150 &  0.9702 & 0.9961 & 1.0000 \\
  200 &  0.9707 & 0.9964  & 1.0000 \\
  Bayes & 0.9725 & 0.9973 & 1.0000 
\end{tabular}
\end{table}


The figure illustrates three cases (dimensionalities $P = 4, 10, 100$), with CBAUC shown in green, CV-AUC in red and EBAUC in blue. In these cases, the superiority of CBAUC can be clearly seen, in particular with small sample sizes. The high-dimensional case ($P = 100$) has a high true AUC ($\approx 1.0$). On the other hand, this makes the estimation task easier (CBAUC error is very small), but still reveals the variability of the CV-AUC estimator with small number of samples.

\sakira{We also extend the experiment further by studying problems with differing target AUCs and by including sample-binormal-AUC to the studied methods. We consider a two dimensional case with $\mub_\classzero = \mathbf{0}$, $\mub_\classone = \mathbf{1}$, and $\Sigmab_\classzero = \Sigmab_\classone = \sigma \mathbf{I}$, where $\sigma$ is varied to produce different target AUCs. The parameters (means and covariances) for the sample-binormal AUC are estimated from the sample data. Results are shown in Fig.~\ref{fig:target_auc}.} \jussi{The sample-binormal-AUC is a poor choice, especially when the target AUC is low because the estimates are positively biased.}

\sakira{We also address the variability of the errors with respect to training sample drawn, see~\citep{chakraborty2017observer}. In Fig.~\ref{fig:synthetic_acc} (d)-(f) and Fig.~\ref{fig:target_auc} (b), we show the variances of the errors with respect to the true AUC.  Results show that CBAUC is a good choice for a small number of samples, since it has lower variability than the other methods that is a desirable quality.}

\jussi{In order to study the sensitivity of the method to the assumptions (Gaussian class-conditional distributions with equal covariance matrices)}, we test the AUC estimators in a four-dimensional case with $\mub_\classzero = \mathbf{0}$, $\mub_\classone = [-1.5, -0.75, 0.75, 1.5]^T$, $\Sigmab_\classzero = \mathbf{I}$ and  $\Sigmab_\classone = \textrm{diag}(0.25, 0.75, 1.25, 1.75)$---as in ~\citep{xue2015does}. \jussi{This is intended to address the assumption that class covariances are equal of both classes.} 
The results are shown in Fig.~\ref{fig:synthetic_acc_with_diff_cov}. It is clear that also in this case the cross-validation performance is poorer than that of the Bayesian approach; in particular when the number of training samples is small. On the other hand, the empirical Bayesian AUC seems to be less sensitive to the violation of the assumptions than the closed-form Bayesian AUC. The reason for this is clear: the CBAUC is fully Bayesian and only views the data through their statistical properties (means and covariances), while the empirical AUC uses the actual sample values while iterating the threshold (and is thus "less Bayesian"). \jussi{Given the Gaussianity assumption, the assumption of equal class covariances can be justified as the linear classifier is the Bayes classifier only if the class covariances are equal \cite{Duda01}. }
\vskip -0.2in
We also study the effect of class imbalance on the performance of the three AUC estimators using the simulated dataset. Fig.~\ref{fig:imbalanced_case} summarizes the performance of the estimators with respect to different ratio of class imbalances in training data. The imbalance ratio is the number of observations in minority class to the total number of observations.  We see a detrimental effect on the estimators with the increasing extent of class imbalance in training data, particularly, in $P \gg N$ setting. The performance of the CV-AUC and EBAUC deteriorates with both the highest and the lowest imbalanced ratio. On the other hand, CBAUC remains stable, particularly if we have enough data to represent the features in high dimension.




\subsubsection{Ovarian Cancer Data}
\label{subsec:ocdata_accuracy}

Next, we study the accuracies of the three AUC estimators with the ovarian cancer dataset as a function of the training set size. Results are summarized in Fig.~\ref{fig:ovarian_acc} and show that CBAUC has the least MAE to the test AUC (AUC computed from the hold-out test partition). The EBAUC estimator converges to CBAUC as the number of training samples increases (Fig.~\ref{fig:ovarian_acc} blue curve). On the other hand, CV-AUC requires more training samples ($> 100$), and still does not reach similar estimation accuracy as the Bayesian approaches. Since the dimension of the data ($P = 4000$) is quite large compared to the number of samples ($N = 20, 40, ..., 100$), CV-AUC is clearly a poor choice for evaluating a classifier in this $P \gg N$ case.

\begin{figure}[htb]
\begin{center}
\includegraphics[width=4cm]{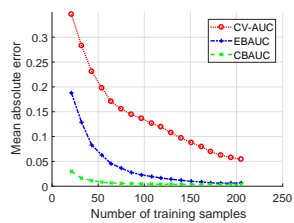} 
\caption{Accuracy assessment using ovarian cancer dataset. Mean absolute error (MAE) of three AUC estimators to the test AUC. The error estimators are 5-fold cross-validation AUC (CV-AUC) (red curve), empirical Bayesian AUC (EBAUC) (blue curve) and proposed closed-form Bayesian AUC (CBAUC) (green curve). }
\label{fig:ovarian_acc}
\end{center}
\end{figure}

\subsubsection{ICANN 2011 MEG Dataset}
\label{subsec:megbindata_accuracy}

Finally, we study the accuracy in a case with moderate dimensionality and number of samples ($N\approx P$): the ICANN 2011 MEG dataset. Unlike the ovarian cancer experiment, this dataset has a predefined train/test split, and we can compare the AUC estimates with those calculated from the test set; simulating the original competition situation, where the teams had to assess how well different strategies work on the hidden test samples. The results of this experiment averaged over 1000 iterations for each $N$ are illustrated in Fig.~\ref{fig:acc_diff_megbin_case}. 

The results show that in this case, the CV-AUC estimator outperforms the Bayesian estimators for most sample sizes. Moreover, the proposed Bayesian approach has the worst performance among the three estimators. Also, the CV-AUC estimator accuracy has an obscure shape suggesting that the estimator would be the most accurate with about 150 training samples but degrade in accuracy as more samples are added. In fact, this is a consequence of the discrepancy between the train/test subsets: The two sets were in fact measured from the same subject but in different days. As a result, the measurement setting or the mental state of the subject was not exactly the same, and the characteristics of the two measurements are different. Moreover, the organizers intentionally added a small amount of second day (test day) samples to the training set (first day). Since the CV-AUC is based purely on looking at the data, it can exploit these samples, while they are only harmful for the Bayesian estimators as the training set becomes a mixture of two Gaussians violating the assumptions of the model.

\begin{figure}[htb]
\begin{center}
\subfigure[]{\includegraphics[width=4cm]{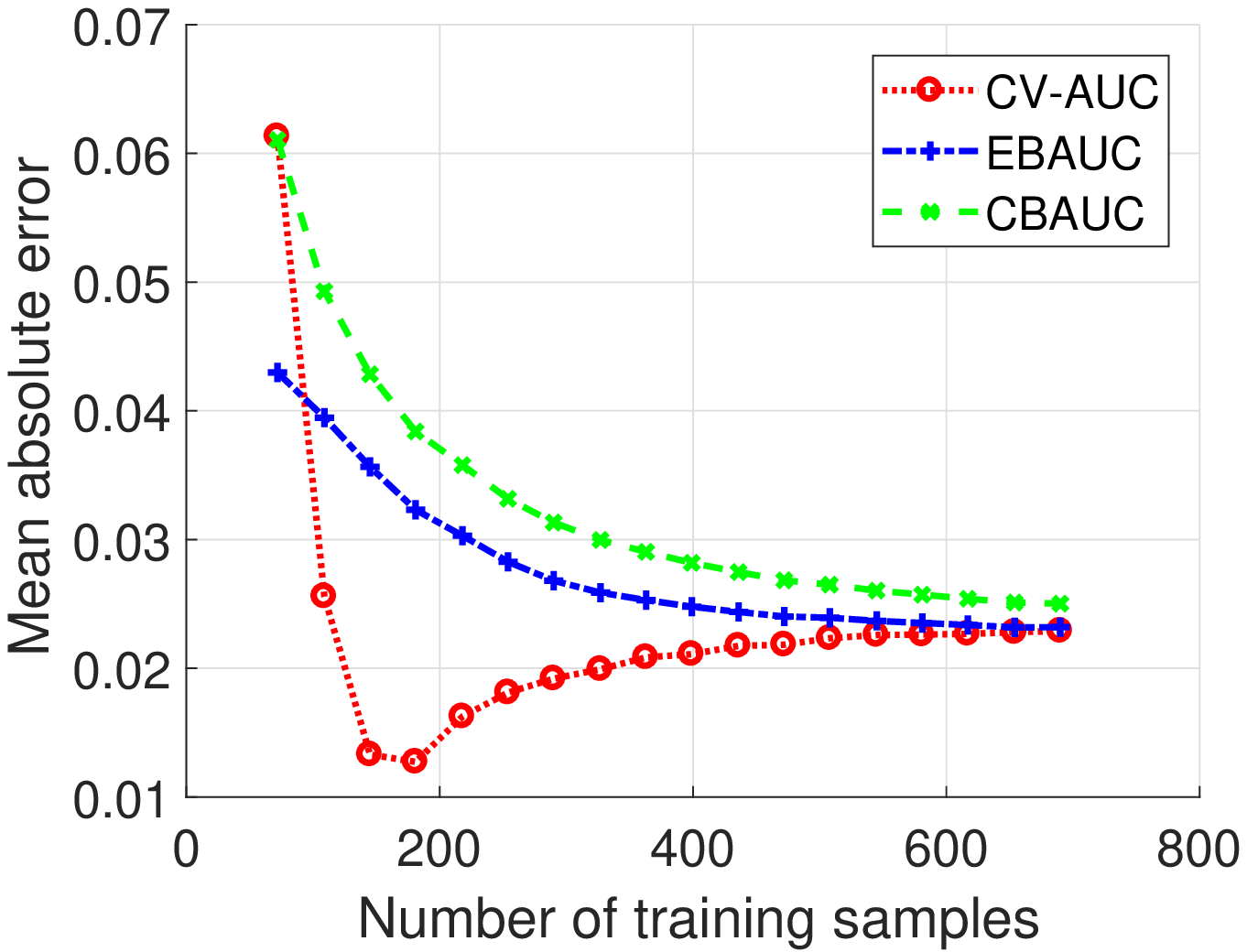} \label{fig:acc_diff_megbin_case}}
\subfigure[]{\includegraphics[width=4cm]{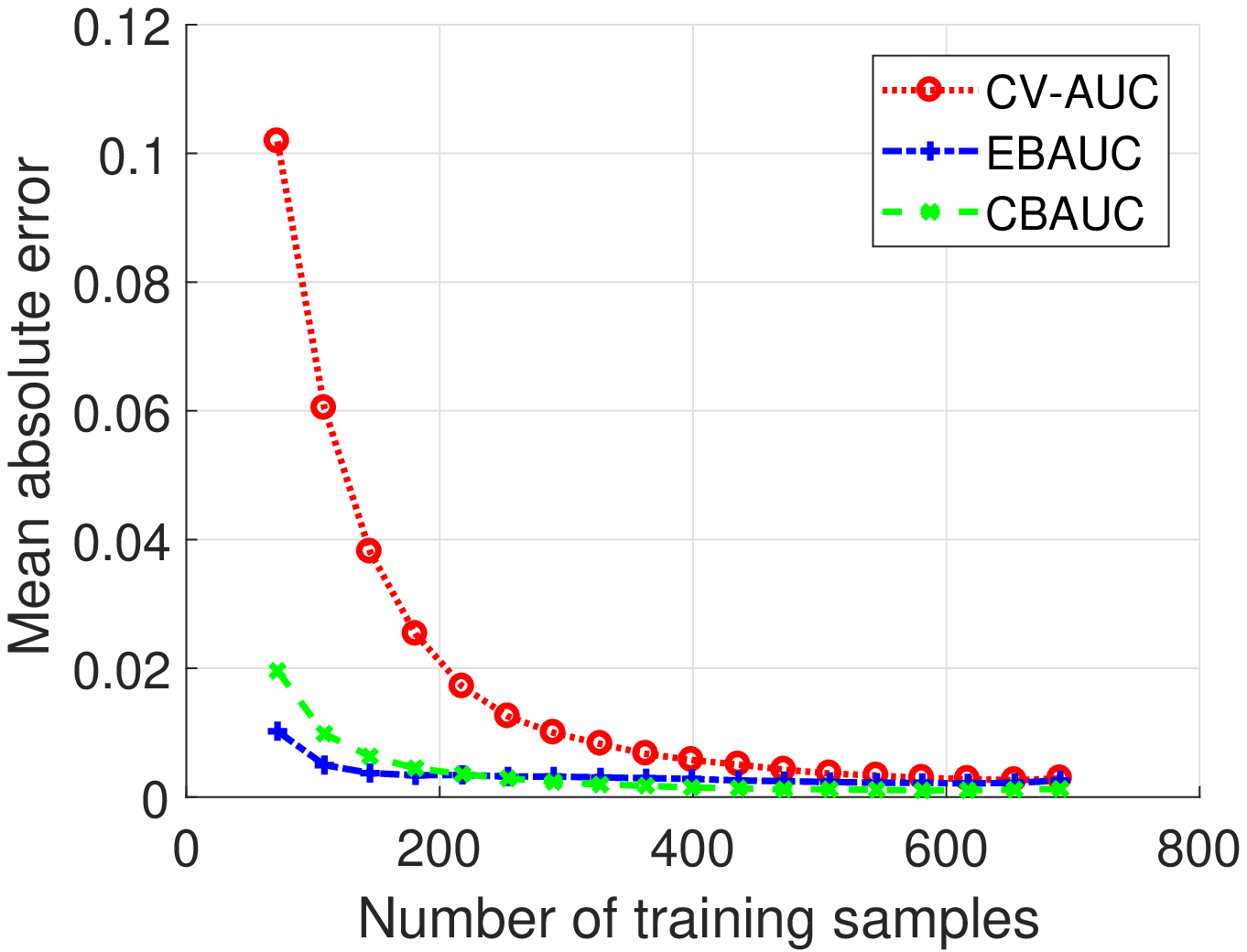}\label{fig:acc_diff_megbin_tr_case}}
\caption{Accuracy assessment using ICANN 2011 MEG binary dataset. Mean absolute error (MAE) of three AUC estimators to the test AUC (a) using the predefined train/test split and (b) using hold-out test data within the training set. The error estimators are 5-fold cross-validation AUC (CV-AUC) (red curve), empirical Bayesian AUC (EBAUC) (blue curve) and proposed closed-form Bayesian AUC (CBAUC) (green curve). }
\label{fig:mebin_test_acc}
\end{center}
\end{figure}

\vskip -0.2in
\section {Conclusions}
\label{sec:conclusions}

In this paper, we introduced a new classifier accuracy metric, CBAUC, which estimates the AUC of a linear classifier. The estimator has a closed-form solution that can be quickly evaluated using only the training data. This speeds up the evaluation by avoiding computationally expensive cross-validation splits, but also improves the accuracy and stability as long as the model assumptions are valid. 

We also studied the accuracy of the proposed estimator using simulated and real-world data. Results showed that the accuracy of estimator is superior with small number of training samples provided that prior assumptions of the underlying distribution hold. We recognized that conventional counting based approaches---such as cross-validation, bootstrapping, etc.--- 	will outperform the proposed method if the model assumptions fail, or if there is discrepancy between the training and test partitions. As an example of the former case, we experimented with data from non-Gaussian distributions, and the Bayesian estimator indeed failed to accurately estimate the AUC.


If we look more thoroughly into the error of the proposed estimator, we can see that the Bayesian approach tends to have a higher bias but smaller variance than the counting based approaches. In particular, the proposed approach tends to over-estimate the AUC, but its small variance more than compensates for the bias and leads to an accurate AUC estimate. More importantly, the bias is usually less important than the variance in the most important use case of model selection. In this case, the primary goal is to select the most accurate model and the accurate measurement of the performance is only secondary. In the future, it would be of interest to further explore the proposed estimator for imbalanced distribution of prior class probabilities.
\section*{Acknowledgments}

This work was partially funded by the Academy of Finland
\jussi{projects 309903 CoefNet and 316258 PredBrain}. Authors also thank CSC--IT Center for Science for computational resources.






\bibliographystyle{elsarticle-num}
\bibliography{hehu}



\end{document}